\documentclass{article}

    \PassOptionsToPackage{numbers, compress}{natbib}
    \usepackage[final]{neurips_2021}

\usepackage[normalem]{ulem}
\usepackage[utf8]{inputenc} % allow utf-8 input
\usepackage[T1]{fontenc}    % use 8-bit T1 fonts
\usepackage{hyperref}       % hyperlinks
\usepackage{url}            % simple URL typesetting
\usepackage{booktabs}       % professional-quality tables
\usepackage{amsfonts}       % blackboard math symbols
\usepackage{nicefrac}       % compact symbols for 1/2, etc.
\usepackage{microtype}      % microtypography
\usepackage{tablefootnote}
\usepackage{footmisc}
\usepackage{graphicx}
\usepackage{bbm}
\usepackage[ruled,vlined]{algorithm2e}
\usepackage{color}
\SetKwInput{KwInput}{Input}

\usepackage[dvipsnames]{xcolor}
\usepackage{amsmath,amsthm}

\newtheorem{assumption}{Assumption}[section]
\newtheorem{theorem}{Theorem}[section]

\newtheorem{lemma}[theorem]{Lemma}

\newcommand{\Ac}{\mathcal{A}}

\newcommand{\Ec}{\mathcal{E}}
\newcommand{\Fc}{\mathcal{F}}
\newcommand{\Gc}{\mathcal{G}}

\newcommand{\Nc}{\mathcal{N}}
\newcommand{\Oc}{\mathcal{O}}

\newcommand{\Sc}{\mathcal{S}}

\newcommand{\Pb}{\mathbb{P}}

\newcommand{\Eb}{\mathbb{E}}
\newcommand{\Rb}{\mathbb{R}}

\newcommand{\argmax}{\arg\max}

\usepackage[flushleft]{threeparttable}
\usepackage{footnote}
\makesavenoteenv{tabular}
\makesavenoteenv{table}

\title{Learning Policies with Zero or Bounded Constraint Violation for Constrained MDPs}

\author{
  Tao Liu\thanks{The first two authors contributed equally.} \\
  Texas A\&M University\\
  \texttt{tliu@tamu.edu} \\
  \And
  Ruida Zhou$^*$ \\
  Texas A\&M University \\
  \texttt{ruida@tamu.edu} \\
  \And
  Dileep Kalathil \\
  Texas A\&M University \\
  \texttt{dileep.kalathil@tamu.edu} \\
  \And
  P. R. Kumar \\
  Texas A\&M University \\
  \texttt{prk@tamu.edu} \\
  \And
  Chao Tian \\
  Texas A\&M University \\
  \texttt{chao.tian@tamu.edu} \\
}

\begin{document}

\maketitle

\begin{abstract}
We address the issue of safety in reinforcement learning. We pose the problem in an episodic framework of a constrained Markov decision process. Existing results have shown that it is possible to achieve a reward regret of $\tilde{\mathcal{O}}(\sqrt{K})$ while allowing an $\tilde{\mathcal{O}}(\sqrt{K})$ constraint violation in $K$ episodes. A critical question that arises is whether it is possible to keep the constraint violation even smaller. We show that when a strictly safe policy is known, then one can confine the system to zero constraint violation with arbitrarily high probability while keeping the reward regret of order $\tilde{\mathcal{O}}(\sqrt{K})$. The algorithm which does so employs the principle of optimistic pessimism in the face of uncertainty to achieve safe exploration. When no strictly safe policy is known, though one is known to exist, then it is possible to restrict the system to bounded constraint violation with arbitrarily high probability. This is shown to be realized by a primal-dual algorithm with an optimistic primal estimate and a pessimistic dual update.
\end{abstract}

\section{Introduction}
Reinforcement learning (RL) addresses the problem of learning an optimal control policy that maximizes the expected cumulative reward while interacting with an unknown environment  \cite{sutton2018reinforcement}. Standard RL algorithms typically focus only on maximizing a single objective. However, in many real-world applications, the control policy learned by an RL algorithm has to additionally satisfy stringent safety constraints \cite{dulac2019challenges, amodei2016concrete}. For example, an autonomous vehicle may need to reach its destination in the minimum possible time without violating safety constraints such as crossing the middle of the road. The Constrained Markov Decision Process (CMDP) \cite{altman1999constrained, garcia2015comprehensive} formalism, where one seeks to maximize a reward while satisfying safety constraints, is a standard approach for modeling the necessary safety criteria of a control problem via constraints on cumulative costs. 

Several policy-gradient-based algorithms have been proposed to solve CMDPs. Lagrangian-based methods \cite{tessler2018reward,stooke2020responsive, paternain2019safe,liang2018accelerated} formulate the CMDP problem as a saddle-point problem and optimize it via primal-dual methods, while Constrained Policy Optimization \cite{achiam2017constrained, yang2019projection} (inspired by the trust region policy optimization \cite{schulman2015trust}) computes new dual variables from scratch at each update to maintain constraints during learning. Although these algorithms provide ways to learn an optimal policy, performance guarantees about reward regret, safety violation or sample complexity are rare. 

One class of RL algorithms for which performance guarantees are available follow the principle of Optimism in the Face of Uncertainty (OFU) \cite{ding2021provably,efroni2020exploration,qiu2020upper}, and provide an $\tilde{\Oc}(\sqrt{K})$ guarantee for the reward regret, where $K$ is the number of episodes. However, these algorithms also have $\tilde{\Oc}(\sqrt{K})$ safety violations. Such significant violation of the safety constraints during learning may be unacceptable in many safety-critical real-world applications such as the control of autonomous vehicles or power systems. These applications demand a class of safe RL algorithms that can provably guarantee safety during learning. With this goal in mind,  we aim to answer the following open theoretical question in this paper:

\textbf{Can we design safe RL algorithms that can achieve an $\tilde{\Oc}(\sqrt{K})$ regret with respect to the performance objective, while guaranteeing zero or bounded safety constraint violation with arbitrarily high probability?}

\begin{table}[t]
  \caption{Regret and constraint violation comparisons for algorithms on episodic CMDPs} 
  \label{overall}
  \centering 
  \begin{tabular}{lll} 
    \toprule
    Algorithm & Regret\tablefootnote{This table is presented for $K \ge \text{poly}(|\Sc|, |\Ac|, H)$, with polynomial terms independent of $K$ omitted.\label{omitpoly}} & Constraint violation \footref{omitpoly}\\
    \midrule
    OPDOP \cite{ding2021provably} & $\tilde{\Oc}(H^3 \sqrt{|\Sc|^2 |\Ac| K})$  & $\tilde{\Oc}(H^3 \sqrt{|\Sc|^2 |\Ac| K})$ \\
    OptCMDP \cite{efroni2020exploration} \tablefootnote{Efroni et al. \cite{efroni2020exploration} use $\Nc$, the maximum number of non-zero transition probabilities across the entire state-action space, in their regret and constraint violation analysis. For consistency, we use $|\Sc|^2 |\Ac|$ to bound $\Nc$. \label{UBN}} & $\tilde{\Oc}(H^2 \sqrt{|\Sc|^3 |\Ac| K})$ & $\tilde{\Oc}(H^2 \sqrt{|\Sc|^3 |\Ac| K})$ \\
    OptCMDP-bonus \cite{efroni2020exploration} \footref{UBN} & $\tilde{\Oc}(H^2 \sqrt{|\Sc|^3 |\Ac| K})$ & $\tilde{\Oc}(H^2 \sqrt{|\Sc|^3 |\Ac| K})$ \\
    OptDual-CMDP \cite{efroni2020exploration} \footref{UBN} & $\tilde{\Oc}(H^2 \sqrt{|\Sc|^3 |\Ac| K})$ & $\tilde{\Oc}(H^2 \sqrt{|\Sc|^3 |\Ac| K})$ \\
    OptPrimalDual-CMDP \cite{efroni2020exploration} \footref{UBN} & $\tilde{\Oc}(H^2 \sqrt{|\Sc|^3 |\Ac| K})$ & $\tilde{\Oc}(H^2 \sqrt{|\Sc|^3 |\Ac| K})$ \\
    C-UCRL \cite{zheng2020constrained} \tablefootnote{Zheng et al. \cite{zheng2020constrained} assumes a known transition kernel and analyzes regret in the long-term average setting.} & $\tilde{\Oc}(T^{\frac{3}{4}})$ & $0$ \\
    \midrule
    \textbf{OptPess-LP} & $\tilde{\Oc}(\frac{H^3}{\tau - c^0} \sqrt{|\Sc|^3 |\Ac| K})$ & $0$ \\
    \textbf{OptPess-PrimalDual} & $\tilde{\Oc}(\frac{H^3}{\tau - c^0} \sqrt{|\Sc|^3 |\Ac| K})$ & $\Oc(1)$ \tablefootnote{The detailed constraint violation is $\Oc\left(C''H + H^2 \sqrt{|\Sc|^3 |\Ac| C'' \log(C''/\delta')} \right)$, which is independent of $K$. Here, $\delta' = \delta / (16 |\Sc|^2 |\Ac| H)$ \mbox{ and } $C'' = \Oc (\frac{H^4 |\Sc|^3 |\Ac|}{(\tau - c^0)^2} \log\frac{H^4 |\Sc|^3 |\Ac|}{(\tau - c^0)^2\delta'})$.}\\
    \bottomrule
  \end{tabular}
\end{table}    

We answer the above question affirmatively by proposing two algorithms and establishing their stringent safety performance during learning. Our focus is on the tabular episodic constrained RL setting (unknown transition probabilities, rewards, and costs). The key idea behind both algorithms is a concept used earlier for safe exploration in constrained bandits \cite{pacchiano2021stochastic,liu2021efficient}, which we call ``Optimistic Pessimism in the Face of Uncertainty (OPFU)'' here. The optimistic aspect incentivizes the algorithm for using exploration policies that can visit new state-action pairs, while the pessimistic aspect disincentivizes the algorithm from using exploration policies that can violate safety constraints. By carefully balancing optimism and pessimism, the proposed algorithms guarantee zero or bounded safety constraint violation during learning while achieving an $\tilde{\Oc}(\sqrt{K})$ regret with respect to the reward objective.

The two algorithms address two different classes of the safe learning problem: whether a strictly safe policy is known a priori or not. The resulting exploration strategies are very different in the two cases.

\begin{enumerate}
    \item \textbf{OptPess-LP Algorithm:} This algorithm assumes the prior knowledge of a strictly safe policy. It ensures zero safety constraint violation during learning with high probability and utilizes the linear programming (LP) approach for solving a CMDP problem. The algorithm achieves a reward regret of $\tilde{\Oc}(\frac{H^3}{\tau - c^0} \sqrt{|\Sc|^3 |\Ac| K})$ with respect to the performance objective, where $H$ is the number of steps per episode, $\tau$ is the given constraint on safety violation, $c^0$ is the known safety constraint value of a strictly safe policy $\pi^0$, and $|\Sc|$ and $|\Ac|$ are the number of states and actions respectively.
    \item \textbf{OptPess-PrimalDual Algorithm:} This algorithm addresses the case where no strictly safe policy, but a feasible strictly safe cost is known. By allowing a bounded (in $K$) safety cost, it opens up space for exploration. The OptPess-PrimalDual algorithm avoids linear programming and its attendant complexity and exploits the primal-dual approach for solving a CMDP problem. The proposed approach improves the computational tractability, while ensuring a bounded safety constraint violation during learning and a reward regret of $\tilde{\Oc}(\frac{H^3}{\tau - c^0} \sqrt{|\Sc|^3 |\Ac| K})$ with respect to the objective.
\end{enumerate}

Compared with the other methods listed in Table \ref{overall}, though the proposed algorithms have an additional $H/(\tau - c^0)$ or $\sqrt{|\Sc|}/(\tau-c^0)$ factor in the regret bounds, they are able to reduce the constraint violation to zero or constant with high probability. This improvement in safety can be extremely important for many mission-critical applications.

\subsection{Related work}
The problem of learning an optimal control policy that satisfies safety constraints has been studied both in the RL setting and the multi-armed bandits setting. 

\paragraph{Constrained RL:}
Several policy-gradient algorithms have seen success in practice \cite{tessler2018reward,stooke2020responsive,paternain2019safe,liang2018accelerated,achiam2017constrained,yang2019projection}. Also of interest are works which utilize Gaussian processes to model the transition probabilities and value functions \cite{berkenkamp2017safe,wachi2018safe,koller2018learning,cheng2019end}.  

Several algorithms with provable guarantees are closely related to our work. Zheng et al. \cite{zheng2020constrained} considered the constrained RL problem in an infinite horizon setting with unknown reward and cost functions. The approach is similar to the UCRL2 algorithm \cite{jaksch2010near} and achieves a sub-linear reward regret $\tilde{\Oc}(T^{\frac{3}{4}})$ while satisfying constraints with high probability during learning. In contrast to this work, we consider the setting of unknown transition probabilities. Efroni et al \cite{efroni2020exploration} focused on the episodic setting of unknown non-stationary transitions over a finite horizon, attaining both a reward regret and a constraint violation of $\tilde{\Oc}(H^2 \sqrt{|\Sc|^3 |\Ac| K})$. Ding et al. \cite{ding2021provably} studied an episodic setting with linear function approximation (suitable for large state space cases), and proposed algorithms  that can achieve $\tilde{\Oc}(d H^3 \sqrt{K})$ for both the regret and the constraint violation (where $d$ is the dimension of the feature mapping). The regret analysis in \cite{ding2021provably} can be easily extended to the tabular case, yielding $\tilde{\Oc}(H^3 \sqrt{|\Sc|^2 |\Ac| K})$.

\paragraph{Constrained Multi-Armed Bandits:}
Multi-armed bandit problems are special cases of MDPs, with both the number of states as well as the episode length being one. Linear bandits with constraints (satisfied with high probability) have been investigated in different settings. One setting, referred to as conservative bandits \cite{wu2016conservative, kazerouni2017conservative, garcelon2020improved}, requires the cumulative reward to remain above a fixed percentage of the cumulative reward of a given baseline policy. Another setting is where each arm is associated with two unknown distributions (similar to our setting), generating reward and cost signals respectively \cite{amani2019linear,pacchiano2021stochastic,liu2021efficient,moradipari2019safe}.

\section{Problem formulation}
\label{sec:prob_formula}
A finite-horizon constrained non-stationary MDP model is defined as a tuple $M = (\Sc, \Ac, H, P, r, c, \tau, \mu)$, where $\Sc$ is the state space, $\Ac$ is the action space, $H$ is the number of steps in each episode, $r: \Sc \times \Ac \rightarrow [0, 1]$ is the unknown reward function of interest, $c: \Sc \times \Ac \rightarrow [0, 1]$ is the unknown safety cost function used to model the constraint violation, $\tau \in (0, H]$ is the given constant used to define the safety constraint, and $\mu$ is the known initial distribution of the state. $P_{\cdot} (\cdot | s, a) \in \Delta_{\Sc}^H, \forall s \in \Sc, \forall a \in \Ac$, where $\Delta_{\Sc}$ is the $|\Sc|$-dimensional probability simplex, and $P_{h} (s' | s, a)$ is the unknown transition probability that the next state is $s'$ when action $a$ is taken for state $s$ at step $h$. 

Some further notation is necessary to define the problem. A (randomized Markov) policy is defined by a map $\pi: \Sc \times [H] \rightarrow \Delta_{\Ac}$, with $\pi_h(a | s)$ being the probability of taking action $a$ in state $s$ at time step $h$. With $S_t$ and $A_t$ representing the state and the action at time $t$ respectively, let   
\begin{align*}
V^{\pi}_{h}(s; g, P) := \Eb_{P, \pi}\left[ \sum_{t = h}^{H} g(S_t, A_t) | S_h = s\right], \quad \forall s \in \Sc
\end{align*}
denote the expected cumulative value with respect to a function $g: \Sc \times \Ac \rightarrow \Rb_+$ under $P$ for a policy $\pi$ over a time interval $[h,h+1,\ldots,H]$. With slightly abuse of notation, we use $V^{\pi}_1(\mu;g, P)$ to denote $\Eb_{S_1 \sim \mu}[V^{\pi}_1(S_1; g, P)]$.

In the formulation below there are a total of $K$ episodes with $H$ steps each.
Each episode $k \in [K]$ begins with an initial probability distribution $\mu$ for $\Sc_1$.
Then, the agent determines a randomized Markov policy $\pi^k$ for that episode based on the information gathered from the previous episodes, and executes it. At time step $h$ during the execution of the $k$-th episode, after taking action $A_h^k$ at state $S_h^k$, the agent receives
a noisy reward and cost of $R_h^k(S_h^k, A_h^k) = r_h(S_h^k, A_h^k) + \xi^k_{h}(S_h^k, A_h^k; r)$ and $C_h^k(S_h^k, A_h^k) = c_{h}(S_h^k, A_h^k) + \xi^k_h(S_h^k, A_h^k; c)$, respectively.
\begin{assumption}[Sub-Gaussian noise]
    \label{assump:noise}
    For all $h \in [H], k \in [K]$, the reward and cost noise random variables are conditionally independent zero-mean 1/2-sub-Gaussian, i.e., $\Eb[\xi_h^k|\Fc_{k-1}] = 0$, $\Eb[\exp(\lambda \xi^k_h) | \Fc_{k-1}] \le \exp(\lambda^2/4)$, $\forall \lambda \in \Rb$. Here $\Fc_{k}$ is the $\sigma$-algebra generated by the random variables up to episode $k$.
\end{assumption}

Let $\pi^*$ denote the optimal policy of the following CMDP model: 
\begin{align}
\max_{\pi} & \quad V_{1}^{\pi}(\mu; r, P)  \quad\quad \text{s.t.} \quad V_{1}^{\pi}(\mu; c, P) \leq \tau.
\label{eqn:cmdp}
\end{align}

A policy $\pi$ is said to be \emph{strictly safe} if $V_1^{\pi}(\mu;c, P) < \tau$.

\begin{assumption}
    There exists a strictly safe policy $\pi^0$ with $V^{\pi^0}_{1}(\mu; c, P) = c^0 < \tau$.
    \label{assumption_safe}
\end{assumption}

There are two important cases of the safe learning problem.

\noindent\textbf{Zero constraint violation case}: 
The agent has prior knowledge of a strictly safe policy $\pi^0$ and its safety cost value $c^0 := V_1^{\pi^0}(\mu; c, P)$. The agent wishes to attain a sublinear (in $K$) \emph{cumulative regret},
\begin{align}
    Reg(K; r) := \sum_{k = 1}^K \left( V_{1}^{\pi^*}(\mu; r, P) - V^{\pi^k}_{1}(\mu; r, P) \right),
    \label{eqn::regret}
\end{align}
while incurring zero constraint violation with at least a specified high probability $(1-\delta)$, i.e.,
\begin{align*}
    \Pb\left( V^{\pi^k}_1(\mu;c , P) \leq \tau, \forall k \in [K] \right) \geq 1- \delta.
\end{align*}

\noindent\textbf{Bounded constraint violation case}:
The agent knows that there exists a strictly safe policy with a known safety cost value $c^0$, but does not know any strictly safe policy. It aims to achieve a cumulative regret (\ref{eqn::regret}) that grows sublinearly with $K$, while ensuring that the \emph{regret of constraint violation},
\begin{align*}
        Reg(K; c) := \left( \sum_{k = 1}^{K} \left( V^{\pi^k}_{1}(\mu; c, P) -  \tau \right) \right)_{+} \mbox{  (where $(a)_{+} := \max\{a, 0\}$) },
\end{align*}
satisfies $\sup_{K} Reg(K;c) < +\infty$ with at least a specified high probability $(1 - \delta)$. 

\noindent\textbf{Remarks.} 
For the zero constraint violation case, the assumption of knowing $c^0$ can be relaxed, as shown in Appendix E.

\section{Zero constraint violation case} 
\label{sec:zero}
We start by considering the zero constraint violation case. On one hand, to balance the exploration-exploitation trade-off, we employ an optimistic estimate of the reward function, as embodied in the OFU principle. On the other hand, to maintain absolute safety with high probability during the exploration, we employ a pessimistic estimate of the safety cost. Such an OPFU principle was previously discussed in constrained bandits \cite{amani2019linear} and we adapt it to the unknown CMDP setting.

At each episode $k$, we begin by forming empirical estimates of the transition probabilities, the reward function, and the cost function from step $h$ of all previous episodes:
\begin{align}
    \hat{P}_h^k(s'|s, a) &:= \frac{\sum_{k'=1}^{k-1} \mathbbm{1}(S_h^{k'}=s, A_h^{k'}=a, S_{h+1}^{k'}=s')}{N^k_h(s, a) \vee 1}, \label{eqn:hatP}\\
    \hat{g}_h^k(s, a) &:= \frac{\sum_{k'=1}^{k-1} \mathbbm{1}(S_h^{k'}=s, A_h^{k'}=a) (g_h(s, a) + \xi_h^k(s, a; g))}{N^k_h(s, a) \vee 1}, \quad \mbox{ for } g = r, c,
    \label{estimate2}
\end{align}
where $a \vee b := \max\{a, b\}$, and
\begin{align}
    N^k_h(s, a) := \mbox{\# of  visits to state-action pair $(s, a)$ at step $h$ in episodes $[1,2,\ldots, k-1]$.}
    \label{eqn:bigN}
\end{align}

Next we fix some $\delta \in (0, 1)$ and form a common (for notational simplicity) confidence radius $\beta^k_h(s, a)$ for the transition probabilities, the rewards, and the costs, 
\begin{align*}
    \beta^k_h(s, a) &:= \sqrt{\frac{1}{N^k_h(s, a) \vee 1} Z}, \quad \text{where $Z := \log(16|\Sc|^2 |\Ac|HK/\delta)$.}
\end{align*}

At each episode $k$, we define the \emph{optimistically biased reward estimate} as
\begin{align}
    \bar{r}^k_h(s, a) &:= \hat{r}^k_h(s, a) + \alpha_r \beta^k_h(s, a), \quad \forall (s,a, h) \in \Sc \times \Ac \times [H], \label{eqn:tilde_r1}
\end{align}
where the scaling factor is
\begin{align}
  &\alpha_r := 1 + |\Sc| H + \frac{4 H(1+|\Sc|H)}{\tau - c^0}. \label{alphar}
\end{align}

To guarantee safe exploration, define the \emph{pessimistically biased safety cost estimate} at episode $k$ as
\begin{align}
    \underbar{c}{}_h^k(s, a) & := \hat{c}_h^k(s,a) + (1 + H|\Sc|) \beta^k_h(s, a), \quad \forall (s, a, h) \in \Sc \times \Ac \times [H].
\end{align}

The policy we execute at episode $k$ is chosen from a ``pessimistically safe" policy set $\Pi^k$, defined as
\begin{equation}
  \Pi^k := 
    \begin{cases}
      \{\pi^0\} & \text{if } V_1^{\pi^0}(\mu; \underbar{c}^k, \hat{P}^k) \geq (\tau + c^0)/2,\\
      \{\pi: V_1^{\pi}(\mu; \underbar{c}^k, \hat{P}^k) \leq \tau\} & \text{otherwise.} 
    \end{cases}
    \label{eqn:PiPC}
\end{equation}

We simply use the strictly safe policy $\pi^0$ until $V_1^{\pi^0}(\mu; \underbar{c}^k, \hat{P}^k) < (\tau + c^0)/2$, which is a sufficient condition to guarantee that the set $\{\pi: V_1^{\pi}(\mu; \underbar{c}^k, \hat{P}^k) \leq \tau\}$ is non-empty. Within this set, we choose the optimistically best reward earning policy $\pi^k$, which can be solved by linear programming with $\Theta(|\Sc||\Ac|H)$ decision variables and constraints \cite{efroni2020exploration}. The resulting Optimistic Pessimism-based Linear Programming (OptPess-LP) algorithm is presented below:
\begin{algorithm}
    \label{OptPess}
    \SetAlgoLined
    \KwInput{$K, \delta, \pi^0, c^0, \tau$;}
    \noindent\textbf{Initialization:} $N_h^1(s, a) = 0, \ \forall (s, a, h) \in \Sc \times \Ac \times [H]$;\\
    \For{$k = 1,2, \ldots, K$}{
        Update empirical model (i.e., $\hat{P}^k, \hat{r}^k, \hat{c}^k$) as in Equations (\ref{eqn:hatP})-(\ref{eqn:bigN});\\
        Update $\bar{r}^k, \underbar{c}^k$, and $\Pi^k$ as in Equations (\ref{eqn:tilde_r1})-(\ref{eqn:PiPC});\\
        Calculate $\pi^k \in \argmax_{\pi \in \Pi^k} V^{\pi}_1(\mu; \bar{r}^k, \hat{P}^k)$;\\
        Execute $\pi^k$ and collect a trajectory $(S_h^k, A_h^k, R_h^k, C_h^k), \ \forall h \in [H]$;\\
        Update counters $N_h^{k+1}(S_h^k, A_h^k), \ \forall h \in [H]$; 
    }
    \caption{\textbf{OptPess-LP}}
\end{algorithm}
 
\begin{theorem}[Regret and constraint violation bounds for OptPess-LP]
\label{thm:safe-PC}
    Fix any $\delta \in (0, 1)$. With probability at least ($1 - \delta$), OptPess-LP has zero constraint violation with
    \begin{align*}
        Reg^{\textbf{OPLP}}(K; r) = \tilde{\Oc}\left(\frac{H^3}{\tau - c^0} \sqrt{|\Sc|^3 |\Ac| K} +  \frac{H^5 |\Sc|^3 |\Ac|}{(\tau - c^0)^2 \land (\tau - c^0)}  \right), \mbox{ where $a \land b := \min\{a, b\}$.}
    \end{align*}
    \label{thm1} 
\end{theorem}

Theorem \ref{thm1} shows that it is possible to achieve sublinear regret in $K$, while simultaneously incurring no constraint violation with arbitrarily high probability. The proof is sketched in Section \ref{zero_sketch}, with detailed proofs presented in Appendix B.

\section{Bounded constraint violation case}
\label{sec:bounded}

Without prior knowledge of a strictly safe policy $\pi^0$, we may not be able to guarantee zero constraint violation with high probability. However, by relaxing the requirement to bounded (in $K$) safety constraint violation, we can incorporate more exploration and design a more efficient algorithm by a primal-dual approach. It is inspired by the design of the pessimistic term in constrained bandits \cite{liu2021efficient}.

Different from traditional optimistic dual analysis methods in \cite{efroni2020exploration}, we introduce an additive pessimistic term $\epsilon_k$ at each episode $k$ in the original optimization problem (\ref{eqn:cmdp}), i.e., 
\begin{align}
    \max_{\pi} & \quad V_1^{\pi}(\mu; r, P) \quad\quad \text{s.t.}\quad V_1^{\pi}(\mu; c, P) + \epsilon_k \le \tau.
    \label{pertubed_problem}
\end{align}
 
The pessimistic term restrains the constraint violation and will be progressively decreased as learning proceeds.
 
The CMDP problem in (\ref{pertubed_problem}) may however not have any feasible solution. To overcome this, we consider the Lagrangian,
\begin{align*}
    L^k(\pi, \lambda) := V^{\pi}_1(\mu;r, P) + \lambda \left( \tau - \epsilon_k - V_1^{\pi}(\mu;c, P) \right),
\end{align*}
for which, given any Lagrange multiplier $\lambda$, we can always solve for the optimizer $\max_{\pi} L^k(\pi, \lambda)$ by dynamic programming. 

With the introduction of the additive term $\epsilon_k$, the dual variable (denoted by $\lambda^k$) governed by the subgradient algorithm grows faster, which enhances safety constraints in the next episode. 

In order to guarantee sufficient optimism of rewards and costs, we integrate the uncertainty of transitions into rewards and costs, i.e., 
\begin{align}
    \label{empirical}
    \tilde{r}_h^k(s, a) &:= \hat{r}_h^k(s, a) + \beta(n_k(s, a, h)) + H |\Sc| \beta(n_k(s, a, h)), \quad \forall (s, a, h) \in \Sc \times \Ac \times [H], \notag\\
    \tilde{c}_h^k(s, a) &:= \hat{c}_h^k(s, a) - \beta(n_k(s, a, h)) - H |\Sc| \beta(n_k(s, a, h)), \quad \forall (s, a, h) \in \Sc \times \Ac \times [H].
\end{align}

Note that in contrast to the zero constraint violation case, we \emph{optimistically} estimate the safety cost function, with the pessimism only governed by $\epsilon_k$. 

We employ truncated value functions due to the additional uncertainties from transitions, i.e.,
\begin{align}
    \hat{V}_1^{\pi}(\mu; \tilde{r}^k, \hat{P}^k) := \min\{H, V_1^{\pi}(\mu; \tilde{r}^k, \hat{P}^k)\}, ~~
    \hat{V}_1^{\pi}(\mu; \tilde{c}^k, \hat{P}^k) := \max\{0, V_1^{\pi}(\mu; \tilde{c}^k, \hat{P}^k)\}.
    \label{truncated}
\end{align}

For the policy update of the primal variable (denoted by $\pi^k$), we can apply standard dynamic programming by viewing $\tilde{r}_h^k(s, a) - \frac{\lambda^k}{\eta^k}(\tilde{c}_h^k(s, a) - \tau)$ as the reward function. Specifically, we apply backward induction to solve for the optimal policy:
\begin{align*}
    Q_h^k(s, a) = \tilde{r}_h^k(s, a) - \frac{\lambda^k}{\eta^k}(\tilde{c}_h^k(s, a) - \tau) + \sum_{s' \in \Sc} \hat{P}_h^k(s'|s, a) \max_{a'} Q_{h+1}^k(s', a'), \quad \forall h \in [H]
\end{align*}
with $Q_{H+1}^k (s, a) = 0, \forall (s, a) \in \Sc \times \Ac$. Then, $\pi_h^k \in \argmax_{a} Q_h^k(s, a)$, which is computationally efficient (as efficient as policy-gradient-based algorithms in the tabular case). 

The resulting Optimistic Pessimism-based Primal-Dual (OptPess-PrimalDual) algorithm is shown in Algorithm \ref{alg:dual}. It chooses $\epsilon_k := 5 H^2 \sqrt{|\Sc|^3 |\Ac|} (\log\frac{k}{\delta'} + 1) / \sqrt{k \log\frac{k}{\delta'}}$, $\delta' = \delta / (16 |\Sc|^2 |\Ac| H)$, the scaling parameter $\eta^k := (\tau - c^0) H \sqrt{k}$ in the primal policy update, and, for convenience, a step size of 1 for the dual update.
\begin{algorithm}
    \label{alg:dual}
    \SetAlgoLined
    \KwInput{$K, \delta, c^0, \tau$;}%\tc{should $\tau$ also be here?}}
    \noindent\textbf{Initialization:} $N_h^1(s, a) = 0, \ \forall (s, a, h) \in \Sc \times \Ac \times [H]$, $\lambda^1 = 0$;\\
    \For{$k = 1,2, \ldots, K$}{
        Set $\epsilon_k = 5 H^2 \sqrt{|\Sc|^3 |\Ac|} (\log\frac{k}{\delta'} + 1) / \sqrt{k \log\frac{k}{\delta'}}$, $\delta' = \delta / (16 |\Sc|^2 |\Ac| H)$, $\eta^k = (\tau - c^0) H \sqrt{k}$;\\
        Update empirical model (i.e., $\hat{P}^k, \hat{r}^k, \hat{c}^k$) as in Equations (\ref{eqn:hatP})-(\ref{eqn:bigN});\\
        Update $\tilde{r}^k$, $\tilde{c}^k$ as in Equation (\ref{empirical});\\
        \emph{(Policy Update)} $\pi^k \in \argmax_{\pi \in \Pi} V_1^{\pi}(\mu; \tilde{r}^k, \hat{P}^k) - \frac{\lambda^k}{\eta^k} \left(V_1^{\pi}(\mu; \tilde{c}^k, \hat{P}^k) - \tau \right)$;\\
        \emph{(Dual Update)} $\lambda^{k+1} = \left(\lambda^k + \hat{V}_1^{\pi^k}(\mu; \tilde{c}^k, \hat{P}^k) + \epsilon_{k} - \tau\right)_{+}$;\\
        Execute $\pi^k$ and collect a trajectory $(S_h^k, A_h^k, R_h^k, C_h^k), \ \forall h \in [H]$;\\
        Update counters $N_h^{k+1}(S_h^k, A_h^k), \ \forall h \in [H]$;
    }
    \caption{\textbf{OptPess-PrimalDual}}
\end{algorithm}

\begin{theorem}[Regret and constraint violation bounds for OptPess-PrimalDual]
    Fix any $\delta \in (0, 1)$. Then,
    \begin{align*}
        Reg^{\textbf{OPPD}}(K; r) &= \tilde{\Oc}\left(\frac{H^3}{\tau - c^0} \sqrt{|\Sc|^3 |\Ac| K} + \frac{H^5 |\Sc|^3 |\Ac|}{(\tau - c^0)^2}\right),\\
        Reg^{\textbf{OPPD}}(K; c) &= \Oc\left(C''(H-\tau) + H^2 \sqrt{|\Sc|^3 |\Ac| C''} \right) = \Oc(1),
    \end{align*}
    where $C'' = \Oc (\frac{H^4 |\Sc|^3 |\Ac|}{(\tau - c^0)^2} \log\frac{H^4 |\Sc|^3 |\Ac|}{(\tau - c^0)^2\delta'})$ is a coefficient independent of $K$.
    \label{thm2}
\end{theorem}

Theorem \ref{thm2} shows that it is possible to achieve an $\tilde{\Oc}(\sqrt{K})$ reward regret,
while only allowing bounded constraint violation with arbitrarily high probability. Detailed proofs are presented in Section \ref{bounded_sketch} and Appendix C.

\section{Regret and constraint violation analysis}
\label{proof sketch}

We sketch the key steps in the proofs of Theorems \ref{thm1} and \ref{thm2}. Full details are relegated to Appendices B and C, respectively. Note that our analysis and results are conditioned on the same high probability event (specifically defined in Appendix A), which occurs with probability at least $(1 - \delta)$. 

\subsection{Analysis of OptPess-LP (Algorithm \ref{OptPess})}
\label{zero_sketch}
 
\paragraph{Constraint violation analysis}
The zero constraint violation of OptPess-LP follows from the following property of the pessimistic policy set $\Pi^k$:
\begin{lemma} \label{lem:abs-safe}
With probability at least $(1 - \delta)$, for any $k \in [K]$ and policy $\pi \in \Pi^k$, $V^{\pi}_1(\mu;c, P) \leq \tau$. 
\end{lemma}

\paragraph{Regret of reward analysis} When the parameters $c, P$ are not well estimated, there may not exist any policy $\pi$ such that $V^{\pi}_1(\mu;\underbar{c}^k, \hat{P}^k) \leq \tau$. Hence, as defined in (\ref{eqn:PiPC}), $\Pi^k$ is a singleton set $\{\pi^0\}$, and accordingly $\pi^0$ is executed, even though $\pi^0$ is not safe for $(\underbar{c}^k, \hat{P}^k)$. It subsequently takes several episodes of exploration using $\pi^0$ until it becomes strictly safe for $(\underbar{c}^k, \hat{P}^k)$. At that time, policies close enough to $\pi^0$, of which there are infinitely many, are also safe for $(\underbar{c}^k,\hat{P}^k)$, and so $|\Pi^k| = +\infty$. At this point the learning algorithm can proceed to enhance reward while maintaining safety with respect to $(\underbar{c}^k, \hat{P}^k)$.

To analyze the algorithm, we decompose the reward regret as follows:
\begin{align}
    Reg^{\textbf{OPLP}}(K; r) =& \sum_{k = 1}^K \mathbbm{1}(|\Pi^k| =1) \left(V^{\pi^*}_1(\mu; r, P) - V^{\pi^0}_1(\mu; r, P)\right) \notag \\
    +& \sum_{k = 1}^K \mathbbm{1}(|\Pi^k| > 1) \left( V^{\pi^*}_1(\mu; r, P) - V^{\pi^k}_1(\mu; \bar{r}^k, \hat{P}^k) \right) \notag \\
    +& \sum_{k = 1}^K  \mathbbm{1}(|\Pi^k| >1) \left( V^{\pi^k}_1(\mu; \bar{r}^k, \hat{P}^k) - V^{\pi^k}_1(\mu; r, P)\right) .
    \label{eqn:no_violation_decompose} 
\end{align}

To bound the first term on the right-hand side (RHS) of (\ref{eqn:no_violation_decompose}), we have the
following lemma which gives an upper bound on the number of episodes for exploration by policy $\pi^0$: 
\begin{lemma}\label{lem:burn-in}
With probability at least $(1 - \delta)$, $\sum_{k=1}^K \mathbbm{1}(|\Pi^k| = 1) \leq C'$, where $C' = \tilde{\Oc}(H^4 |\Sc|^3 |\Ac| / ((\tau - c^0)^2 \land (\tau - c^0)))$.
\end{lemma}
 
Turning to the second term on the RHS of (\ref{eqn:no_violation_decompose}), we first note that $\pi^*$ may not be in $\Pi^k$. To ensure that the term $V^{\pi^*}_1(\mu; r, P) - V^{\pi^k}_1(\mu; \bar{r}^k, \hat{P}^k)$ is nevertheless non-positive even when $\pi^* \not\in \Pi^k$, we set $\alpha_r$ to the large value shown in (\ref{alphar}). This increases $\bar{r}^k$, and hence also $V^{\pi^k}_1(\mu; \bar{r}^k, \hat{P}^k)$. In addition, we show that there is a policy $\hat{\pi}^k$ that attains the same reward as the (non-Markov) probabilistic mixed policy, $\tilde{\pi}^k := B_{\gamma_k} \pi^* + (1 - B_{\gamma_k}) \pi^0$, where $B_{\gamma_k}$ is a Bernoulli distributed random variable with mean $\gamma_k$ for $\gamma_k \in [0, 1]$. $\gamma_k$ will be chosen as the largest coefficient such that $V_1^{\tilde{\pi}^k}(\mu; \underbar{c}^k, \hat{P}^k) \le \tau$. This latter policy in turn has a larger reward than $\pi^0$ since it is a mixture with $\pi^*$, yielding the following:
\begin{lemma}
    With probability at least $(1 - \delta)$,
    \begin{align*}
        \sum_{k = 1}^K \mathbbm{1}(|\Pi^k| > 1) \left( V^{\pi^*}_1(\mu; r, P) - V^{\pi^k}_1(\mu; \bar{r}^k, \hat{P}^k)\right) \le 0.
    \end{align*}
    \label{lem1}
\end{lemma}

Finally, concerning the third term
on the RHS of (\ref{eqn:no_violation_decompose}), akin to the closed-loop identifiability property \cite{borkar1979adaptive,kumar1982new},
while $\hat{P}^k$ may not converge to $P$, the difference in the rewards $V^{\pi^k}_1(\mu; \bar{r}^k, \hat{P}^k) - V^{\pi^k}_1(\mu; r, P)$ grows sublinearly in $k$
since the same policy $\pi^k$ is used in both values:
\begin{lemma}
    With probability at least $(1 - \delta)$,
    \begin{align*}
        \sum_{k = 1}^K \mathbbm{1}(|\Pi^k| > 1) \left( V^{\pi^k}_1(\mu; \bar{r}^k, \hat{P}^k) - V^{\pi^k}_1(\mu; r, P)\right) = \tilde{\Oc}\left(\frac{H^3}{\tau - c^0} \sqrt{|\Sc|^3 |\Ac| K} + \frac{H^5 |\Sc|^3 |\Ac|}{\tau-c^0} \right).
    \end{align*}
    \label{lem2}
\end{lemma}

Combining Lemmas \ref{lem:burn-in}, \ref{lem1}, and \ref{lem2} yields Theorem \ref{thm1}.

\subsection{Analysis of OptPess-PrimalDual (Algorithm \ref{alg:dual})}
\label{bounded_sketch}
In this section, we outline the steps in the proof of Theorem \ref{thm2} by analyzing regret and constraint violation of OptPess-PrimalDual respectively.

Recall $\epsilon_k = 5 H^2 \sqrt{|\Sc|^3 |\Ac|} (\log\frac{k}{\delta'} + 1) / \sqrt{k \log\frac{k}{\delta'}}$, where $\delta' = \delta / (16 |\Sc|^2 |\Ac| H)$. The existence of a feasible solution to (\ref{pertubed_problem}) can be guaranteed if $\epsilon_k \le \tau - c^0$. Let $C''$ be the smallest value such that $\forall k \geq C''$, $\epsilon_k \le (\tau - c^0)/2$. Then the perturbed optimization problem (\ref{pertubed_problem}) has at least one feasible solution for any $k \ge C''$. By simple calculation, one can verify that $C'' = \Oc (\frac{H^4 |\Sc|^3 |\Ac|}{(\tau - c^0)^2} \log\frac{H^4 |\Sc|^3 |\Ac|}{(\tau - c^0)^2\delta'})$. Notice that $\epsilon_k$ is a function not depending on $K$, and so is the coefficient $C''$.

\paragraph{Constraint violation analysis}
The bounded constraint violation of OptPess-PrimalDual is established as follows. We first decompose the constraint violation as
\begin{align}
    Reg^{\textbf{OPPD}}(K; c) &= \left(\sum_{k=1}^K \left(V_1^{\pi^k}(\mu; c, P) - \hat{V}_1^{\pi^k}(\mu; \tilde{c}^k, \hat{P}^k)\right) + \sum_{k=1}^K \left(\hat{V}_1^{\pi^k}(\mu; \tilde{c}^k, \hat{P}^k) - \tau\right)\right)_{+} \notag\\
    & \le \left(\sum_{k=1}^K \left(V_1^{\pi^k}(\mu; c, P) - \hat{V}_1^{\pi^k}(\mu; \tilde{c}^k, \hat{P}^k)\right) + \lambda^{K+1} - \sum_{k=1}^K \epsilon_{k} \right)_{+}.
    \label{decompose_bounded_violation}
\end{align}
The first summation term and $\lambda^{K+1}$ in (\ref{decompose_bounded_violation}) can be bounded as follows:
\begin{lemma}
    Recall $\delta' = \delta / (16 |\Sc|^2 |\Ac| H)$, with probability at least $(1 - \delta)$,
    \begin{align*}
        \sum_{k=1}^K \left(V_1^{\pi^k}(\mu;c, P) - \hat{V}_1^{\pi^k}(\mu; \tilde{c}^k, \hat{P}^k)\right) \le 8 H^2 \sqrt{|\Sc|^3 |\Ac| K \log \frac{K}{\delta'}} + \Oc(PolyLog(K)).
    \end{align*}
    \label{lem:violation_first_term}
\end{lemma}

\begin{lemma}
    For any $k \ge C''$, with probability at least $(1 - \delta)$,
    \begin{align*}
        \lambda^k \le \frac{1}{\zeta} \ln \frac{11 \nu_{\max}^2}{3 \rho^2} + C''(H-\tau) + \sum_{u=1}^{C''} \epsilon_u + H + \frac{4(H^2 + \epsilon_k^2 + 2 \eta^k H)}{\tau - c^0},
    \end{align*}
    where $\rho = (\tau - c^0)/4$, $\nu_{\max} = H$, $\zeta = \rho / (\nu_{\max }^{2}+\nu_{\max} \rho / 3)$.
    \label{lem7}
\end{lemma}
To guarantee bounded violation, we ensure that $\sum_{k=1}^K \epsilon_k$ in (\ref{decompose_bounded_violation}) can cancel the dominant terms in the two lemmas above. According to Lemmas \ref{lem:violation_first_term} and \ref{lem7}, with probability at least $(1 - \delta)$, the violation is bounded as
\begin{align*}
    Reg^{\textbf{OPPD}}(K; c) = \Oc\left(C''H + H^2 \sqrt{|\Sc|^3 |\Ac| C'' \log{(C'' / \delta')}} \right). 
\end{align*}

\paragraph{Regret of reward analysis}
For episode $k$ with $k \geq C''$, let $\pi^{\epsilon_k, *}$ be the optimal policy for (\ref{pertubed_problem}), which is well-defined by the definition of $C''$. We decompose the reward regret as
\begin{align}
    Reg^{\textbf{OPPD}}&(K; r) = \sum_{k=1}^{C''} \left(V_1^{\pi^*}(\mu; r, P) - V_1^{\pi^k}(\mu; r, P)\right) \label{bound_regret_decompose}\\
    + &\sum_{k=C''}^K \left(V_1^{\pi^*}(\mu; r, P) - V_1^{\pi^{\epsilon_k, *}}(\mu; r, P)\right) + \sum_{k=C''}^K \left(V_1^{\pi^{\epsilon_k, *}}(\mu; r, P) - \hat{V}_1^{\pi^{\epsilon_k, *}}(\mu; \tilde{r}^k, \hat{P}^k)\right) \notag\\
    + &\sum_{k=C''}^K \left(\hat{V}_1^{\pi^{\epsilon_k, *}}(\mu; \tilde{r}^k, \hat{P}^k) - \hat{V}_1^{\pi^k}(\mu; \tilde{r}^k, \hat{P}^k)\right) + \sum_{k=C''}^K \left(\hat{V}_1^{\pi^k}(\mu; \tilde{r}^k, \hat{P}^k) - V_1^{\pi^k}(\mu; r, P)\right). \notag
\end{align}

We upper bound each term on the RHS of (\ref{bound_regret_decompose}). Since $V_1^{\pi}(\mu;r, P) \in [0, H]$ for any policy $\pi$, the first term is upper bounded by $HC''$. The second and third terms can be bounded by the following two lemmas:
\begin{lemma}
     With probability at least $(1 - \delta)$,
     \begin{align*}
         \sum_{k=C''}^K \left(V_1^{\pi^*}(\mu; r, P) - V_1^{\pi^{\epsilon_k, *}}(\mu; r, P)\right) \le \sum_{k=C''}^K \frac{\epsilon_k H}{\tau - c^0} = \tilde{\Oc}\left(\frac{H^3}{\tau - c^0} \sqrt{|\Sc|^3 |\Ac| K}\right).
     \end{align*}
    \label{lem3}
\end{lemma}

\begin{lemma}
    With probability at least $(1 - \delta)$,
    $\sum_{k=C''}^K \left(V_1^{\pi^{\epsilon_k, *}}(\mu; r, P) - \hat{V}_1^{\pi^{\epsilon_k, *}}(\mu; \tilde{r}^k, \hat{P}^k)\right) \le 0$.
    \label{lem4}
\end{lemma}

The pivotal step is to leverage optimism of $\pi^k$ to further decompose the fourth term on the RHS of (\ref{bound_regret_decompose}), and utilize the projected dual update to transfer it into the form of $\lambda^k (\lambda^k - \lambda^{k+1})$. The following lemmas provide high probability bounds for the remaining two terms on the RHS of (\ref{bound_regret_decompose}):

\begin{lemma}
    With probability at least $(1 - \delta)$,
    \begin{align*}
        \sum_{k=C''}^K \left(\hat{V}_1^{\pi^{\epsilon_k, *}}(\mu; \tilde{r}^k, \hat{P}^k) - \hat{V}_1^{\pi^k}(\mu; \tilde{r}^k, \hat{P}^k)\right) = \tilde{\Oc}\left( \frac{H}{\tau - c^0} \sqrt{K}\right).
    \end{align*}
    \label{lem5}
\end{lemma}

\begin{lemma}
    With probability at least $(1 - \delta)$, 
    \begin{align*}
        \sum_{k=C''}^K \left(\hat{V}_1^{\pi^k}(\mu; \tilde{r}^k, \hat{P}^k) - V_1^{\pi^k}(\mu; r, P)\right) = \tilde{\Oc}\left(H^2 \sqrt{|\Sc|^3 |\Ac| K} + H^4 |\Sc|^3 |\Ac|\right).
    \end{align*}
    \label{lem6}
\end{lemma} 

Applying Lemmas \ref{lem3}, \ref{lem4}, \ref{lem5}, and \ref{lem6} yields Theorem \ref{thm2}.

\section{Concluding remarks}
\label{conclusion}
We present two optimistic pessimism-based algorithms that maintain stringent safety constraints (either zero or bounded safety constraint violation) for unknown CMDPs with high probability, while still attaining an $\tilde{\Oc}(\sqrt{K})$ regret of reward. The algorithms employ, respectively, a pessimistically safe policy set $\Pi^k$ or an additional pessimistic term $\epsilon_k$ into the safety constraint. 
The first algorithm, OptPess-LP, guarantees zero violation with high probability by solving a linear programming with $\Theta(|\Sc||\Ac|H)$ decision variables, while the second algorithm, OptPess-PrimalDual, is as efficient as policy-gradient-based algorithms in the tabular case, but violates constraints during initial episodes. A possible future direction for exploration is the application of the above OPFU principle in model-free algorithms.

\section*{Acknowledgement}
P. R. Kumar's work is partially supported by US National Science Foundation under CMMI-2038625, HDR Tripods CCF-1934904; US Office of Naval Research under N00014-21-1-2385; US ARO under W911NF1810331, W911NF2120064; and U.S. Department of Energy's Office of Energy Efficiency and Renewable Energy (EERE) under the Solar Energy Technologies Office Award Number DE-EE0009031. The views expressed herein and conclusions contained in this document are those of the authors and should not be interpreted as representing the views or official policies, either expressed or implied, of the U.S. NSF, ONR, ARO, Department of Energy  or the United States Government. The U.S. Government is authorized to reproduce and distribute reprints for Government purposes notwithstanding any copyright notation herein.

Dileep Kalathil gratefully acknowledges funding from the U.S. National Science Foundation (NSF) grants NSF-CRII- CPS-1850206 and NSF-CAREER-EPCN-2045783.

\bibliographystyle{abbrv}

\newpage
\appendix
{\Large \textbf{Appendix}}
\vspace{0.2cm}

The theorems and lemmas presented in the paper are provided with full details in this appendix.

First, let's recall some notations. Fix some $0 < \delta < 1$ as the input of algorithms. We say $g = \tilde{\Oc}(f)$ if there exists a universal constant $C$ such that $g \le C \left(f\log\frac{1}{\delta} + f \log f \right)$ for any $f$ and $\delta$. The filtration $\{ \Fc_{k} \}_{k \geq 0}$ is defined as follows: $\Fc_{0}$ is the trivial sigma algebra, and for each $k \in [K]$, $\Fc_{k} = \sigma\left( (S_h^{k'}, A_h^{k'}, R_h^{k'}, C_h^{k'})_{h \in [H], k' \in [k]} \right)$. The policy process, i.e., the sequence of deployed policies $\{ \pi^{k} \}_{k \in [K]}$, is a predictable process with respect to the filtration $\{\Fc_{k}\}_{k \geq 0}$. According to the definition of $N^k_h$, we know $N^k_h(s, a) \in \Fc_{k-1}$. 

Additionally, we define expectation operator $\Eb_{\mu', P', \pi'}[\cdot]$ as the expectation with respect to a stochastic trajectory $(S_{h}, A_{h})_{h \in [H]}$ generated according to the Markov chain induced by $(\mu', P', \pi')$. When $\mu', P', \pi'$ are random elements, $\Eb_{\mu', P', \pi'}[\cdot]$ will be a $\sigma(\mu', P', \pi')$-measurable random variable.

\section{High probability good event $\Ec$}
\label{app:good_event}
We aim to give performance guarantees of the algorithms with high probability. To this end, we first consider a high probability ``good event" $\Ec$ that all the subsequent analysis is conditioned on.

First, for any predictable event sequence $\Gc_{1:K}$, i.e., $\Gc_k \in \Fc_{k-1},~\forall k \in [K]$, define the event 
\begin{align*}
    & \Ec_{\Gc}(\delta) :=  {\Bigg \{}\forall K' \in [K],~ \sum_{k = 1}^{K'} \sum_{h =1}^H \sum_{s, a} \frac{\mathbbm{1}(\Gc_k)q^{\pi^k}(s, a, h)}{N^k_h(s, a) \lor 1}  \leq 4H |\Sc||\Ac| + 2H|\Sc||\Ac|\ln K'_{\Gc} + 4\ln\frac{2HK}{\delta},  \notag \\
    &  \sum_{k = 1}^{K'} \sum_{h =1}^H \sum_{s, a} \frac{\mathbbm{1}(\Gc_k) q^{\pi^k}(s, a, h)}{\sqrt{N^k_h(s, a) \lor 1}} \leq 6H |\Sc||\Ac| + 2H \sqrt{|\Sc||\Ac| K'_{\Gc}} + 2H |\Sc||\Ac|\ln K'_{\Gc} + 5\ln\frac{2HK}{\delta} {\Bigg \}},
\end{align*}
where $K'_{\Gc} := \sum_{k=1}^{K'}\mathbbm{1}(\Gc_k)$ and $q^{\pi^k}$ is the occupancy measure of policy $\pi^k$, i.e., $q^{\pi^k}(s, a,h) = \Eb_{\mu, P, \pi^k}[\mathbbm{1}(S_h^k=s, A_h^k=a) | \Fc_{k-1}]$. 

A trivial predictable event sequence is $\Gc_{1:k}$ with $\Gc_{k} = \Omega, \forall k \in [K]$, where $\Omega$ is the sample space. Let $\Ec_{\Omega}(\delta)$ be the event defined by this trivial event sequence. 

Consider another event sequence $\Gc'_{1:K} = \{V^{\pi^0}_1(\mu; \underbar{c}^k, \hat{P}^k) \geq \frac{\tau + c^0}{2} \}_{k \in [K]}$, which is predictable with respect to $\{\Fc_{k}\}_{k\geq 0}$. Let $\Ec_{0}(\delta)$ be the event $\Ec_{\Gc'}(\delta)$ defined by this event sequence $\Gc'_{1:K}$. Notice that all notations (including $\hat{r}, \hat{c}, \hat{P},$ and $\underbar{c}^k$) are defined in the same manner in the two algorithms, so this event sequence $\Gc'_{1:K}$ is also well-defined in Algorithm \ref{alg:dual}.

\paragraph{Good event $\Ec$.} Define a ``good event" $\Ec$ as
\begin{align}
    \Ec &:= {\Big \{} \forall k \in [K], \forall h \in [H], \forall s \in \Sc, \forall a \in \Ac, \notag \\
    &\quad\quad \left.~ |r_h(s, a)-\hat{r}^k_h(s, a)| \leq \beta^k_h(s, a),~|c_h(s, a) - \hat{c}^k_h(s, a)| \leq \beta^k_h(s, a), \notag \right.\\
    &\quad\quad \left.~ |\hat{P}^k_h(s' | s, a) - P_h(s' | s,a )| \leq \beta^k_h(s, a), \forall s' \in \Sc \right. \notag \\
    & \quad\quad \left.~ |\hat{P}^k_h(s' | s, a) - P_h(s' | s,a )| \leq \tilde{\beta}^k_h(s' | s, a), \forall s' \in \Sc  \right\}\cap \Ec_{\Omega}(\delta/4) \cap \Ec_{0}(\delta/4),
    \label{good event}
\end{align}
where $\tilde{\beta}^k_h(s' | s, a) := \sqrt{\frac{2 P(s' | s, a)}{N^k_h(s, a) \lor 1} Z} + \frac{Z}{3 N^k_h(s, a) \lor 1}$ and $Z := \log(16|\Sc|^2 |\Ac|HK/\delta)$.

\begin{lemma}\label{lem:good-event}
Fix any $\delta \in (0, 1)$ as the confidence parameter in the inputs of the proposed algorithms. The good event $\Ec$ occurs with probability at least $1 - \delta$.
\end{lemma}
\begin{proof}[Proof of Lemma \ref{lem:good-event}]
For each $(s,a,h) \in \Sc \times \Ac \times [H]$, we take $K$ mutually independent samples of the reward, cost, and next state with the probability distribution specified by the generative model $M$:
\begin{align*}
    \{ R^n(s, a, h), C^n(s, a, h), S^n(s, a, h) \}_{n = 1}^K,
\end{align*}
Let $(\hat{r}^n(s, a, h), \hat{c}^n(s, a, h), \hat{P}^n( \cdot |s, a, h))$ be the corresponding (running) empirical means, respectively, for the samples
\begin{align*}
\{ R^i(s, a, h), C^i(s, a, h), S^i(s, a, h) \}_{i = 1}^n.
\end{align*}

Define the following failure events:
\begin{align*}
    F_n^r &:= \{\exists s, a, h: |\hat{r}^n (s, a, h) - r_h(s, a)| \ge \beta(n)\}, \\
    F_n^c &:= \{\exists s, a, h: |\hat{c}^n (s, a, h) - c_h(s, a)| \ge \beta(n)\}, \\
    F_n^P &:= \{\exists s, a, s', h: |P_h(s'|s, a) - \hat{P}^n(s'|s, a, h)| \ge \beta(n)\},\\
    \tilde{F}_n^P &:= \{\exists s, a, s', h: |P_h(s'|s, a) - \hat{P}^n(s'|s, a, h)| \ge \tilde{\beta}(n, P_h(s'|s,a))\},
\end{align*}
where $\beta(n) := \sqrt{\frac{1}{n \lor 1}Z}$ and $\tilde{\beta}(n, p) :=  \sqrt{\frac{2 P(s' | s, a)}{n \lor 1} Z} + \frac{Z}{3 n \lor 1}$.

Define the event $\Ec^{gen}$ $$\Ec^{gen} := \left( (\bigcup_{n=1}^K F_n^r \cup F_n^c \cup F_n^P \cup \tilde{F}_n^P)^C \cap \Ec_{\Omega}(\delta/4) \cap \Ec_{0}(\delta/4) \right).$$ By the definition of $\Ec_{\Omega}(\delta / 4)$ and $\Ec_0(\delta/4)$ and Lemma D.5, event $\Ec_{\Omega}(\delta / 4) \cap \Ec(\delta/4)$ occurs with probability at least $1 - \delta/2$. Therefore, to show that $\Pb(\Ec^{gen}) \geq 1 - \delta$, it is sufficient to show that $\Pb (\bigcup_{n=1}^K F_n^r \cup F_n^c \cup F_n^P \cup \tilde{F}_n^P) \leq \delta/2$. Note that $\delta / 16 \le \delta |\Sc| / 4(1 + |\Sc|) =: \delta'$. Now, it is straightforward to show the following:

(i) Using Hoeffding’s inequality, $ \Pb(\bigcup_{n=1}^K F_n^r) \leq |\Sc||\Ac|HK \frac{\delta}{16 |\Sc|^2 |\Ac|HK} \leq \frac{\delta'}{|\Sc|}$.

(ii) Using Hoeffding’s inequality, $ \Pb(\bigcup_{n=1}^K F_n^c) \leq |\Sc||\Ac|HK \frac{\delta}{16 |\Sc|^2 |\Ac|HK} \leq \frac{\delta'}{|\Sc|}$.

(iii) Using Hoeffding’s inequality, $\Pb(\bigcup_{n=1}^K F_n^P) \leq |\Sc|^2|\Ac|HK \frac{\delta}{16 |\Sc|^2 |\Ac|HK} \leq \delta'$.

(iv) Using Bernstein's inequality, $\Pb(\bigcup_{n=1}^K \tilde{F}_n^P) \leq |\Sc|^2|\Ac|HK \frac{\delta}{16 |\Sc|^2 |\Ac|HK} \leq \delta'$.

Using union bound, $\Pb (\bigcup_{n=1}^K F_n^r \cup F_n^c \cup F_n^P \cup \tilde{F}_n^P) \le \Pb(\bigcup_{n=1}^K F_n^r) + \Pb(\bigcup_{n=1}^K F_n^c) + \Pb(\bigcup_{n=1}^K F_n^P) + \Pb(\bigcup_{n=1}^K \tilde{F}_n^P) \le (2 + 2 / |\Sc|) \delta' = \delta/2$.

The episodic CMDP problem studied here can be simulated as follows: at episode $k$, taking action $a$ at state $s$ and step $h$ returns 
\begin{align*}
    \left(R^{n_k(s, a, h)}(s, a, h), C^{n_k(s, a, h)}(s, a, h), S^{n_k(s, a, h)}(s, a, h) \right).
\end{align*}
Then all the samples drawn in the episodic CMDP problem are contained in
\begin{align*}
    \{ R^n(s, a, h), C^n(s, a, h), S^n(s, a, h) \}_{n = 1}^K.
\end{align*}
The sample averages calculated by the algorithm are
\begin{align*}
    \left(\hat{r}^{k}_h(s, a), \hat{c}^k_h(s, a), \hat{P}^k_h(\cdot | s, a) \right) = \left( \hat{r}^{n_k(s,a,h)}(s, a, h), \hat{c}^{n_k(s,a,h)}(s, a, h), \hat{P}^{n_k(s,a,h)}( \cdot|s, a, h) \right).
\end{align*}

Since $\Ec^{gen}$ implies $\Ec$, $\Ec$ holds with probability at least $1- \delta$. 
\end{proof}

\noindent\textbf{Remark.} \emph{The analysis in the rest of this appendix is conditioned on the good event $\Ec$. That is, if $(\Omega, {\cal{F}}, P)$ is the underlying probability space, then we suppose throughout the rest of the appendix that the sample point $\omega \in \Ec$.} 

\section{Details of analysis for the zero constraint violation case} \label{app:absolute}
\subsection{Constraint violation analysis}
We will provide a proof of a slightly stronger result. Specifically, we will suppose that we know a policy $\pi^0$ and a constant $c^0 < \tau$ such that the safety cost $V_1^{\pi^0}(\mu; c,P) \leq c^0$. The strengthening lies in the relaxed requirement of knowing only an \emph{upper bound} on the safety cost of $\pi^0$, rather than its \emph{exact} value. 

The zero constraint violation of OptPess-LP is guaranteed by the pessimistic policy set $\Pi^k$, as stated in the following lemma. Recall
\begin{equation*}
  \Pi^k := 
    \begin{cases}
      \{\pi^0\} & \text{if } V_1^{\pi^0}(\mu; \underbar{c}^k, \hat{P}^k) \geq (\tau + c^0)/2,\\
      \{\pi: V_1^{\pi}(\mu; \underbar{c}^k, \hat{P}^k) \leq \tau\} & \text{otherwise.} 
    \end{cases}
    \label{app:eqn:PiPC}
\end{equation*}

\begin{lemma} [Restatement of Lemma \ref{lem:abs-safe}]
On the good event $\Ec$, for any $k \in [K]$ and policy $\pi \in \Pi^k$, $V^{\pi}_1(\mu;c, P) \leq \tau$. 
\end{lemma}
\begin{proof}[Proof of Lemma \ref{lem:abs-safe}]
Fix $\omega \in \Ec$ throughput. Then, for any $k, s, a, h$, we have
\begin{align*}
    (\underbar{c}_h^k - c_h)(s, a) & = \underbar{c}_h^k(s, a) - \hat{c}_h^k(s, a) + \hat{c}^k_h(s, a) - c_h(s, a) \\
    & = (1+|\Sc|H)\beta^k_h(s, a) + \hat{c}^k_h(s, a) - c_h(s, a) \\
    & \geq (1+|\Sc|H)\beta^k_h(s, a) - \beta^k_h(s, a) = |\Sc|H\beta^k_h(s, a),
\end{align*}
and 
\begin{align*}
    \sum_{s'} (\hat{P}^k_h - P_h)(s' | s, a)V^{\pi}_{h+1}(s'; c, P) \geq -H \sum_{s'} \beta^k_h(s, a) = -|\Sc|H \beta^k_h(s, a).
\end{align*}
Thus, by Lemma \ref{valuediff}, for any policy $\pi$, we have
    \begin{align}
        & V^{\pi}_1(\mu; \underbar{c}^k, \hat{P}^k) - V^{\pi}_1(\mu; c, P) \notag\\
        =& \Eb_{\mu, \hat{P}^k, \pi}\left[ \sum_{h = 1}^H \left( (\underbar{c}_h^k - c_h)(S_h, A_h) + \sum_{s'} (\hat{P}^k_h - P_h)(s' | S_h, A_h) V^{\pi}_{h+1}(s'; c, P) \right) ~{\Big |}~ \Fc_{k-1}\right] \notag\\
        \geq& \Eb_{\mu, \hat{P}^k, \pi} \left[\sum_{h=1}^H |\Sc|H \beta^k_h(S_h, A_h) - |\Sc|H \beta^k_h(S_h, A_h) ~{\Big |}~ \Fc_{k-1}\right] \geq 0.
    \label{eqn:underbar_c}
    \end{align}
    Consider any $\pi \in \Pi^k$. If $\pi = \pi^0$, then $V^{\pi^0}_1(\mu; c, P) = c^0 < \tau$. Otherwise,
    \begin{align*}
        V_1^{\pi}(\mu; c, P) \le V_1^{\pi}(\mu; \underbar{c}^k, \hat{P}^k) \le \tau.
    \end{align*}
\end{proof}

\subsection{Regret of reward analysis} 
When the parameters $c, P$ are not well estimated, there may not exist any policy $\pi$ such that $V^{\pi}_1(\mu;\underbar{c}^k, \hat{P}^k) \leq \tau$.
Hence, as defined in (\ref{app:eqn:PiPC}), $\Pi^k$ is a singleton set $\{\pi^0\}$, and accordingly $\pi^0$ is executed, even though $\pi^0$ is not safe for $(\underbar{c}^k, \hat{P}^k)$. It subsequently takes several episodes of exploration using $\pi^0$ until it becomes strictly safe for $(\underbar{c}^k, \hat{P}^k)$, which means that $\pi^k = \pi^0$ and $V_1^{\pi^k}(\mu; \underbar{c}^k, \hat{P}^k) \geq (\tau + c^0)/2$ when $|\Pi^k|=1$. At that time, policies close enough to $\pi^0$, of which there are infinitely many, are also safe for $(\underbar{c}^k,\hat{P}^k)$, and so $|\Pi^k| = +\infty$. At this point the learning algorithm can proceed to enhance reward while maintaining safety with respect to $(\underbar{c}^k, \hat{P}^k)$.

To analyze the algorithm, we decompose the reward regret as follows:
\begin{align*}
    Reg(K; r) =& \underbrace{\sum_{k = 1}^K \mathbbm{1}(|\Pi^k| =1) \left(V^{\pi^*}_1(\mu; r, P) - V^{\pi^0}_1(\mu; r, P) \right)}_{\textnormal{(I)}} \\
    +& \underbrace{ \sum_{k = 1}^K \mathbbm{1}(|\Pi^k| > 1) \left( V^{\pi^*}_1(\mu; r, P) - V^{\pi^k}_1(\mu; \tilde{r}^k, \hat{P}^k) \right)}_{\textnormal{(II)}} \\
    +& \underbrace{\sum_{k = 1}^K  \mathbbm{1}(|\Pi^k| >1) \left( V^{\pi^k}_1(\mu; \tilde{r}^k, \hat{P}^k) - V^{\pi^k}_1(\mu; r, P)\right) }_{\textnormal{(III)}}.
\end{align*}

We analyze these three terms, and the results are summarized in the following lemmas, respectively.

The following lemma gives an upper bound on the number of episodes for exploration by policy $\pi^0$. It thereby bounds the first term on the right-hand side (RHS) of (\ref{eqn:no_violation_decompose}).
\begin{lemma}[Restatement of Lemma \ref{lem:burn-in}]
On the good event $\Ec$, $\sum_{k=1}^K \mathbbm{1}(|\Pi^k| = 1) \leq C'$, where $C' = \tilde{\Oc}(H^4 |\Sc|^3 |\Ac| / ((\tau - c^0)^2 \land (\tau - c^0)))$.
\end{lemma}
\begin{proof}[Proof of Lemma \ref{lem:burn-in}]
With $K' := \sum_{k=1}^K \mathbbm{1}(|\Pi^k| = 1)$, we then have 
    \begin{align*}
        \frac{(\tau - c^0) K'}{2} &= \sum_{k = 1}^K \mathbbm{1}(|\Pi^k| =1) \frac{(\tau - c^0)}{2} \\
        & = \sum_{k =1}^{K} \mathbbm{1}(|\Pi^k| =1) \left( \frac{\tau + c^0}{2} - c^0\right) \\
        & \stackrel{(a)}{\leq} \sum_{k = 1}^{K} \mathbbm{1}(V^{\pi^k}_1(\mu; \underbar{c}^k, \hat{P}^k) \geq \frac{\tau + c^0}{2}) \left(V^{\pi^k}_1(\mu; \underbar{c}^k, \hat{P}^k) - V^{\pi^k}_1(\mu; c, P) \right)\\
        & \stackrel{(b)}{=} \tilde{\Oc}\left( H^4 |\Sc|^3 |\Ac| + H^2 \sqrt{|\Sc|^3|\Ac|K'} \right).
    \end{align*}
    (a) holds since $\pi^k = \pi^0$, $V_1^{\pi^k}(\mu; \underbar{c}^k, \hat{P}^k) \geq (\tau + c^0)/2$, and $V^{\pi^{k}}_1(\mu; c, P) \leq c^0$ when $|\Pi^k| = 1$. The equality (b) follows from Lemma \ref{lem:indi-opt-val-diff} with $|\underbar{c}_h^k - c_h| = |\hat{c}_h^k - c_h + (1 + H |\Sc|) \beta_h^k| \le (2+|\Sc|H) \beta_h^k$. Applying Lemma \ref{lem:quad-bound} for $K'$, there exists some parameter $C'$ that
    \begin{align*}
        K' \leq C' = \tilde{\Oc}\left(\frac{H^4 |\Sc|^3 |\Ac|}{(\tau - c^0)\left( 1 \land (\tau - c^0) \right)}  \right).
    \end{align*}
\end{proof}

The following lemma gives the resulting high probability bound for Term (II).
\begin{lemma}[Restatement of Lemma \ref{lem1}]
    Recall $\alpha_r = 1 + |\Sc| H + 4 H(1+|\Sc|H)/(\tau - c^0)$. On the good event $\Ec$,
    \begin{align*}
        \sum_{k = 1}^K \mathbbm{1}(|\Pi^k| > 1) \left( V^{\pi^*}_1(\mu; r, P) - V^{\pi^k}_1(\mu; \bar{r}^k, \hat{P}^k)\right) \le 0.
    \end{align*}
\end{lemma}
\begin{proof}[Proof of Lemma \ref{lem1}]
    Consider any $k \in [K]$ with $V^{\pi^0}_1(\mu; \underbar{c}^k, \hat{P}^k) < (\tau + c^0)/2$.
    Then $|\Pi^k| > 1$. When $\pi^* \in \Pi^k$, it is straightforward to verify that Term (II) is less or equal to zero by optimism. 
    When $\pi^* \not\in \Pi^k$, define a probabilistic mixed policy
    \begin{align*}
        \tilde{\pi}^k = B_{\gamma_k} \pi^* + (1 - B_{\gamma_k}) \pi^0,
    \end{align*}
    where $B_{\gamma_k}$ is an independent Bernoulli distributed random variable with mean $\gamma_k$. Let $\gamma_k \in [0, 1]$ be the largest coefficient such that
    \begin{align}
        V^{\tilde{\pi}^k}_1(\mu; \underbar{c}^k, \hat{P}^k) \leq \tau. \label{temp1}
    \end{align}
    If $V^{\pi^*}_1(\mu; \underbar{c}^k, \hat{P}^k) < \tau$, then $\gamma_k = 1$. Otherwise, at $\gamma_k$, equality holds in (\ref{temp1}). Then we have
    \begin{align*}
         \tau &= \gamma_k V^{\pi^*}_1(\mu;  \underbar{c}^k, \hat{P}^k) + (1 - \gamma_k) V^{\pi^0}_1(\mu; \underbar{c}^k, \hat{P}^k)\\
         & \leq \gamma_k V^{\pi^*}_1(\mu;  \underbar{c}^k, \hat{P}^k) + (1 - \gamma_k) \frac{\tau + c^0}{2} \\
         &= \gamma_k (V_1^{\pi^*}(\mu; \underbar{c}^k, \hat{P}^k) - V_1^{\pi^*}(\mu;c, P)) + \gamma_k V_1^{\pi^*}(\mu;c, P) + (1 - \gamma_k) \frac{\tau + c^0}{2} \\
         & \leq \gamma_k \left(V_1^{\pi^*}(\mu; \underbar{c}^k, \hat{P}^k) - V_1^{\pi^*}(\mu; c, P)\right) + \gamma_k \tau  + \frac{\tau + c^0}{2} - \gamma_k \frac{\tau + c^0}{2}\\
         & = \gamma_k \left(V_1^{\pi^*}(\mu; \underbar{c}^k, \hat{P}^k) - V_1^{\pi^*}(\mu; c, P) + \frac{\tau - c^0}{2} \right) + \frac{\tau + c^0}{2}.
    \end{align*}

    From (\ref{eqn:underbar_c}), $V_1^{\pi^*}(\mu; \underbar{c}^k, \hat{P}^k) - V_1^{\pi^*}(\mu; c, P) + (\tau - c^0)/2 > 0$, from which it follows that
    \begin{align*}
        \gamma_k \geq \frac{\tau - c^0}{\tau - c^0 + 2 \left( V^{\pi^*}_1(\mu; \underbar{c}^k, \hat{P}^k) - V^{\pi^*}_1(\mu; c, P) \right) }.
    \end{align*}
    By Lemma \ref{valuediff}, for any policy $\pi$, 
    \begin{align*}
        V^{\pi}_1(\mu; \underbar{c}^k, \hat{P}^k) - V^{\pi}_1(\mu; c, P) & = \Eb_{\mu, \hat{P}^k, \pi} \left[ \sum_{h=1}^H (\underbar{c}^k_h - c_h)(S_h, A_h) ~{\Big |}~ \Fc_{k-1} \right] \notag \\
        &+ \Eb_{\mu, \hat{P}^k, \pi} \left[ \sum_{h=1}^H \langle (\hat{P}^k_h - P_h)(\cdot | S_h, A_h) , V^{\pi}_{h+1}(\cdot; c, P) \rangle ~{\Big |}~ \Fc_{k-1} \right] \\
        &\stackrel{(a)}{\leq} \Eb_{\mu, \hat{P}^k, \pi} \left[ \sum_{h = 1}^H 2(1 + |\Sc| H) \beta_h^k(S_h, A_h)) ~{\Big |}~ \Fc_{k-1} \right] \\
        & = 2(1+|\Sc|H) V^{\pi}_1(\mu; \beta^k, \hat{P}^k),
    \end{align*}
    where (a) holds since for any $k, s, a, h$, we have
    \begin{align*}
        (\underbar{c}_h^k - c_h)(s, a) &= \underbar{c}_h^k(s, a) - \hat{c}_h^k(s, a) + \hat{c}_h^k(s, a) - c_h(s, a)\\
        & \le (1 + |\Sc|H) \beta_h^k(s, a) + \beta_h^k(s, a)\\
        & = (2 + |\Sc|H) \beta_h^k(s, a),
    \end{align*}
    and 
    \begin{align*}
        \sum_{s'} (\hat{P}^k_h - P_h)(s' | s, a)V^{\pi}_{h+1}(s'; c, P) \leq H \sum_{s'} \beta^k_h(s, a) = |\Sc|H \beta^k_h(s, a).
    \end{align*}

    Though policy $\tilde{\pi}^k$ is not a (randomized or non-randomized) Markov policy, we can find a randomized Markov policy $\hat{\pi}^k$ such that the occupation distributions of the state-action pair at any time of $\hat{\pi}^k$ coincides with $\tilde{\pi}^k$ under the transition probabilities $\hat{P}^k$. Therefore, as stated in by Lemma \ref{lem:richness}, $V^{\hat{\pi}^k}_1(\mu; g, \hat{P}^k) = V^{\tilde{\pi}^k}_1(\mu; g, \hat{P}^k)$ for any $g$.
    
    Since $\hat{\pi}^k \in \Pi^k$, by the definition of $\pi^k$, we then have
    \begin{align*}
        V^{\pi^k}_1(\mu; \bar{r}^k, \hat{P}^k) & \geq V^{\hat{\pi}^k}_1(\mu; \bar{r}^k, \hat{P}^k) = V^{\tilde{\pi}^k}_1(\mu; \bar{r}^k, \hat{P}^k) \\
        & = \gamma_k V^{\pi^*}_1(\mu; \bar{r}^k, \hat{P}^k) + (1 - \gamma_k) V^{\pi^0}_1(\mu; \bar{r}^k, \hat{P}^k) \\
        & \geq \gamma_k V^{\pi^*}_1(\mu; \bar{r}^k, \hat{P}^k) \\
        & = \frac{\tau - c^0}{\tau - c^0 + 2\left( V^{\pi^*}_1(\mu; \underbar{c}^k, \hat{P}^k) - V^{\pi^*}_1(\mu; c, P) \right) } V^{\pi^*}_1(\mu; \bar{r}^k, \hat{P}^k) \\
        & \geq \frac{\tau - c^0}{\tau - c^0 +  4(1+|\Sc|H) V^{\pi^*}_1(\mu; \beta^k, \hat{P}^k)  } V^{\pi^*}_1(\mu; \bar{r}^k, \hat{P}^k).
    \end{align*}
    To make $V^{\pi^k}_1(\mu; \bar{r}^k, \hat{P}^k) \geq V^{\pi^*}_1(\mu; r, P)$, it suffices that
    \begin{align*}
        \frac{\tau - c^0}{\tau - c^0 +  4(1+|\Sc|H)V^{\pi^*}_1(\mu; \beta^k, \hat{P}^k)  } V^{\pi^*}_1(\mu; \bar{r}^k, \hat{P}^k) \geq V^{\pi^*}_1(\mu; r, P),
    \end{align*}
    which follows since
    \begin{align*}
        & (\tau - c^0) (V_1^{\pi^*}(\mu; \bar{r}^k, \hat{P}^k) - V_1^{\pi^*}(\mu; r, P)) \geq 4(1+|\Sc|H)V_1^{\pi^*}(\mu; \beta^k, \hat{P}^k) V^{\pi^*}_1(\mu; r, P).
    \end{align*}
    Now notice that from Lemma \ref{valuediff},
    \begin{align*}
        & V_1^{\pi^*}(\mu; \bar{r}^k, \hat{P}^k) - V_1^{\pi^*}(\mu; r, P) \\
        =& \Eb_{\mu, \hat{P}^k, \pi^*} \left[ \sum_{h = 1}^H \left( (\bar{r}^k_h - r_h)(S_h, A_h) + \sum_{s'} (\hat{P}^k_h - P_h)(s' | S_h, A_h) V^{\pi^*}_{h+1}(s';r, P) \right) ~{\Big |}~ \Fc_{k-1} \right] \notag \\
        \geq& \Eb_{\mu, \hat{P}^k, \pi^*} \left[\sum_{h=1}^H (\alpha_r - 1 - H|\Sc|) \beta_h^k(S_h, A_h) ~{\Big |}~ \Fc_{k-1} \right] \\
        = & (\alpha_r - 1 - H|\Sc|) V^{\pi^*}_1(\mu;\beta^k, \hat{P}^k ).
    \end{align*}

    Thus, the choice of $\alpha_r = 4H(1+|\Sc|H)/(\tau - c^0) + 1 + |\Sc| H$ suffices to guarantee that $V^{\pi^k}_1(\mu; \bar{r}^k, \hat{P}^k) \geq V^{\pi^*}_1(\mu; r, P)$ for any $k$ with $|\Pi^k| > 1$.
\end{proof}

Finally, Term (III) is bounded as follows by employing the value of $\alpha_r$.
\begin{lemma}[Restatement of Lemma \ref{lem2}]
    On the good event $\Ec$,
    \begin{align*}
        \sum_{k = 1}^K \mathbbm{1}(|\Pi^k| > 1) \left( V^{\pi^k}_1(\mu; \bar{r}^k, \hat{P}^k) - V^{\pi^k}_1(\mu; r, P)\right) = \tilde{\Oc}\left(\frac{H^3}{\tau - c^0} \sqrt{|\Sc|^3 |\Ac| K} + \frac{H^5 |\Sc|^3 |\Ac|}{\tau-c^0} \right).
    \end{align*}
\end{lemma}
\begin{proof}[Proof of Lemma \ref{lem2}]
Since $|\bar{r}_h^k - r_h| = |\hat{r}_h^k - r_h + \alpha_r \beta_h^k| \le (1 + \alpha_r) \beta_h^k$, by Lemma \ref{lem:indi-opt-val-diff}, we have
    \begin{align*}
        & \sum_{k = 1}^{K} \mathbbm{1}(|\Pi^k| >1)\left( V^{\pi^k}_1(\mu; \bar{r}^k, \hat{P}^k) - V^{\pi^k}_1(\mu; r, P)\right) \\
        \leq &  \sum_{k=1}^K {\Big |} V^{\pi^k}_1(\mu; \bar{r}^k, \hat{P}^k) - V^{\pi^k}_1(\mu; r, P) {\Big |} \\
        = & \tilde{\Oc}\left( (\alpha_r + H \sqrt{|\Sc|}) H \sqrt{|\Sc| |\Ac| K} + H^3 |\Sc|^2 |\Ac| (\alpha_r + |\Sc| ) \right) \\
        = & \tilde{\Oc}\left( \frac{H^3}{\tau - c^0} \sqrt{|\Sc|^3 |\Ac| K} + \frac{H^5 |\Sc|^3 |\Ac|}{\tau - c^0} \right).
    \end{align*}
\end{proof}

\subsection{Proof of Theorem \ref{thm:safe-PC}}
\begin{theorem}[Regret and constraint violation bounds for OptPess-LP (Restatement of Theorem \ref{thm:safe-PC})]
    On the good event $\Ec$, OptPess-LP has zero constraint violation with
    \begin{align*}
        Reg^{\textbf{OPLP}}(K; r) = \tilde{\Oc}\left(\frac{H^3}{\tau - c^0} \sqrt{|\Sc|^3 |\Ac| K} +  \frac{H^5 |\Sc|^3 |\Ac|}{(\tau - c^0)^2 \land (\tau - c^0)}  \right), \mbox{ where $a \land b := \min\{a, b\}$.}
    \end{align*}
    \label{thm:restate_1}
\end{theorem}

\begin{proof}[Proof of Theorem \ref{thm:safe-PC}]
The zero constraint violation of OptPess-LP is guaranteed by Lemma \ref{lem:abs-safe}.
Combining Lemmas \ref{lem:burn-in}, \ref{lem1}, and \ref{lem2}, which give upper bounds to Terms (I)-(III) respectively, we have
\begin{align*}
    & \sum_{k = 1}^K \left( V^{\pi^*}_1(\mu; r, P) - V^{\pi^k}_1(\mu; r, P) \right) 
    = \tilde{\Oc}\left(\frac{H^3}{\tau - c^0} \sqrt{|\Sc|^3 |\Ac| K} + \frac{H^5 |\Sc|^3 |\Ac|}{(\tau - c^0) \land (\tau - c^0)^2} \right).
\end{align*}
\end{proof}

\section{Details of analysis for bounded constraint violation case}
\label{app:bounded}

Recall $\epsilon_k = 5 H^2 \sqrt{|\Sc|^3 |\Ac|} (\log\frac{k}{\delta'} + 1) / \sqrt{k \log\frac{k}{\delta'}}$, where $\delta' = \delta / (16 |\Sc|^2 |\Ac| H)$. The existence of a feasible solution to (\ref{pertubed_problem}) can be guaranteed if $\epsilon_k \le \tau - c^0$. Let $C''$ be the smallest value such that $\forall k \geq C''$, $\epsilon_k \le (\tau - c^0)/2$. Then the perturbed optimization problem (\ref{pertubed_problem}) has at least one feasible solution for any $k \ge C''$. By simple calculation, one can verify that $C'' = \Oc (\frac{H^4 |\Sc|^3 |\Ac|}{(\tau - c^0)^2} \log\frac{H^4 |\Sc|^3 |\Ac|}{(\tau - c^0)^2\delta'})$. Notice that since $\epsilon_k$ is a function not depending on $K$, so is the coefficient $C''$.

\begin{lemma}[Optimistic of cost value] \label{lem:opt-cost}
On the good event $\Ec$, for any $C'' \leq k \leq K$ and any policy $\pi$,
\begin{align*}
     \hat{V}^{\pi}_1(\mu; \tilde{c}^k, \hat{P}^k) \leq V^{\pi}_1(\mu; c, P).
\end{align*}
\end{lemma}
\begin{proof}[Proof of Lemma \ref{lem:opt-cost}]
If $\hat{V}_1^{\pi}(\mu; \tilde{c}^k, \hat{P}^k) = V_1^{\pi}(\mu; \tilde{c}^k, \hat{P}^k)$, then  
 \begin{align*}
    & \sum_{k = C''}^{K} \left(\hat{V}^{\pi}_1(\mu; \tilde{c}^k, \hat{P}^k) - V^{\pi}_1(\mu; c, P)\right) \\
    =& \sum_{k = C''}^{K} \Eb_{\mu, \hat{P}^k, \pi} \left[ \sum_{h = 1}^H \left(\tilde{c}_h^k(S_h, A_h) - c_h(S_h, A_h) \right) + \langle (\hat{P}_h^k - P_h)(\cdot | S_h, A_h) , V^{\pi}_{h+1}(\cdot; c, P) \rangle ~{\Big |}~ \Fc_{k-1} \right]\\
    \le& \sum_{k=C''}^K \Eb_{\mu, \hat{P}^k, \pi} \left[\sum_{h = 1}^H (\hat{c}_h^k(S_h, A_h) - c_h(S_h, A_h)) - \beta_h^k - H |\Sc| \beta_h^k + H |\Sc| \beta_h^k ~{\Big |}~ \Fc_{k-1} \right] \le 0. 
\end{align*}
Otherwise, $\hat{V}_1^{\pi}(\mu; \tilde{c}^k, \hat{P}^k) = 0 \leq V^{\pi}_1(\mu; c, P)$.
\end{proof}

\subsection{Constraint violation analysis}
\begin{lemma}[Restatement of Lemma \ref{lem:violation_first_term}]
    Recall $\delta' = \delta / (16 |\Sc|^2 |\Ac| H)$. On the good event $\Ec$,
    \begin{align*}
        \sum_{k=1}^K \left( V_1^{\pi^k}(\mu;c, P) - \hat{V}_1^{\pi^k}(\mu; \tilde{c}^k, \hat{P}^k)\right) \le 8 H^2 \sqrt{|\Sc|^3 |\Ac| K \log \frac{K}{\delta'}} + \tilde{\Oc}\left( H^4 |\Sc|^3 |\Ac|\right),
    \end{align*}
    where $\tilde{\Oc}(H^4|\Sc|^3|\Ac|)$ contains factors poly-logarithmic in $K$.
\end{lemma}
\begin{proof}[Proof of Lemma \ref{lem:violation_first_term}]
For any $k, s, a, h$, we have
\begin{align*}
    |(\tilde{c}_h^k - c_h)(s, a)| & = |\tilde{c}_h^k(s, a) - \hat{c}_h^k(s, a) + \hat{c}^k_h(s, a) - c_h(s, a)| \\
    & \leq (1+|\Sc|H)\beta^k_h(s, a) + |\hat{c}^k_h(s, a) - c_h(s, a)| \\
    & \leq (1+|\Sc|H)\beta^k_h(s, a) + \beta^k_h(s, a) = (2+|\Sc|H)\beta^k_h(s, a).
\end{align*}
By Lemma \ref{lem:indi-opt-val-diff}, we have
\begin{align*}
    &\sum_{k=1}^K \left( V_1^{\pi^k}(\mu;c, P) - \hat{V}_1^{\pi^k}(\mu; \tilde{c}^k, \hat{P}^k) \right)\\
    \leq & (3(2+|\Sc|H) + 3\sqrt{2}H\sqrt{|\Sc|})H\sqrt{|\Sc||\Ac| K Z} + \tilde{\Oc}\left( H^4 |\Sc|^3|\Ac|\right) \\
    = & (6 + 3|\Sc|H + 3H\sqrt{2|\Sc|})H\sqrt{|\Sc||\Ac| K Z} + \tilde{\Oc}\left( H^4 |\Sc|^3 |\Ac|\right) \\
    \stackrel{(a)}{\leq} & 8 H^2|\Sc|\sqrt{|\Sc||\Ac|KZ} + \tilde{\Oc}\left( H^4 |\Sc|^3 |\Ac|\right),
\end{align*}
where $\tilde{\Oc}(H^4|\Sc|^3|\Ac|)$ contains factors poly-logarithmic in $K$. (a) holds since we assume that $|\Sc| \ge 2$ and $H \ge 2$.
\end{proof}

\begin{lemma}[Restatement of Lemma \ref{lem7}]
    On the good event $\Ec$, for any $k \ge C''$, 
    \begin{align*}
        \lambda^k \le \frac{1}{\zeta} \ln \frac{11 \nu_{\max}^2}{3 \rho^2} + C''(H-\tau) + \sum_{u=1}^{C''} \epsilon_u + H + \frac{4(H^2 + \epsilon_k^2 + 2\eta^k H)}{\tau - c^0},
    \end{align*}
    where $\rho := (\tau - c^0)/4$, $\nu_{\max} := H$, $\zeta := \rho / (\nu_{\max }^{2}+\nu_{\max} \rho / 3)$. 
\end{lemma}

\begin{proof}[Proof of Lemma \ref{lem7}]
We will first verify conditions (i) and (ii) of Lemma \ref{Lyapunov} for $k \geq C''$ by defining $\Phi(k) = \lambda^k$. Suppose $\Phi(k) = \lambda^k \ge \varphi_k := 4(H^2 + \epsilon_k^2 + 2\eta^k H)/(\tau - c^0)$. It is straightforward to verify that $\lambda^k + \hat{V}_1^{\pi^k}(\mu; \tilde{c}^k, \hat{P}^k) + \epsilon_{k} - \tau$ is positive, which implies
\begin{align*}
    \lambda^{k+1} = \lambda^k + \hat{V}_1^{\pi^k}(\mu; \tilde{c}^k, \hat{P}^k) + \epsilon_{k} - \tau.
\end{align*}

Let $L(k) = (\lambda^k)^2/2$, we have
\begin{align*}
        L(k+1) - L(k) &= \frac{1}{2} (\lambda^{k+1})^2 - \frac{1}{2} (\lambda^k)^2 \\
        & = \lambda^k (\lambda^{k+1} - \lambda^k) + \frac{1}{2}(\lambda^{k+1} - \lambda^k)^2 \\
        & = \lambda^k \left(\hat{V}_1^{\pi^k}(\mu; \tilde{c}^k, \hat{P}^k)  + \epsilon_k - \tau \right) + \frac{1}{2} (\hat{V}_1^{\pi^k}(\mu; \tilde{c}^k, \hat{P}^k) + \epsilon_k - \tau)^2.
    \end{align*}
{\color{black} To bound the first term, consider the following two cases.

If $V_1^{\pi^k}(\mu; \tilde{c}^k, \hat{P}^k) \geq c_0$, we know $V_1^{\pi^k}(\mu; c - \tilde{c}^k, \hat{P}^k) = V_1^{\pi^k}(\mu; c, \hat{P}^k) - V_1^{\pi^k}(\mu; \tilde{c}^k, \hat{P}^k) \leq H - c_0$. Since $c - \tilde{c}^k = \tilde{r}^k - r$, we know $V_1^{\pi^k}(\mu; \tilde{r}^k, \hat{P}^k) = V_1^{\pi^k}(\mu; \tilde{r}^k - r, \hat{P}^k) + V_1^{\pi^k}(\mu; r, \hat{P}^k) \leq 2 H - c_0$. Moreover, $V_1^{\pi^k}(\mu; \tilde{c}^k, \hat{P}^k) \geq c_0$ implies that $\hat{V}_1^{\pi^k}(\mu; \tilde{c}^k, \hat{P}^k) = V_1^{\pi^k}(\mu; \tilde{c}^k, \hat{P}^k)$. We have
\begin{align*}
    & \lambda^k \left(\hat{V}_1^{\pi^k}(\mu; \tilde{c}^k, \hat{P}^k)  + \epsilon_k - \tau \right) \\
    =& \left(\lambda^k \left(V_1^{\pi^k}(\mu; \tilde{c}^k, \hat{P}^k) + \epsilon_k - \tau\right) - \eta^k V_1^{\pi^k}(\mu; \tilde{r}^k, \hat{P}^k) \right) + \eta^k V_1^{\pi^k}(\mu; \tilde{r}^k, \hat{P}^k) \\
    \stackrel{(a)}{\leq}& \left( \lambda^k \left(V_1^{\pi^0}(\mu; \tilde{c}^k, \hat{P}^k) + \epsilon_k - \tau\right) - \eta^k V_1^{\pi^0}(\mu; \tilde{r}, \hat{P}^k) \right) + 2 \eta^k H \\
    \stackrel{(b)}{\leq}& \lambda^k \left(V_1^{\pi^0}(\mu; c, P) + \epsilon_k - \tau\right) + 2 \eta^k H \leq - \frac{\tau - c^0}{2} \lambda^k + 2 \eta^k H,
\end{align*}
where $(a)$ is by the optimality of $\pi^k$ and $V_1^{\pi^k}(\mu; \tilde{r}^k, \hat{P}^k) \leq 2H$, and $(b)$ is due to $V_1^{\pi^0}(\mu; \tilde{c}^k, \hat{P}^k) \leq \hat{V}_1^{\pi^0}(\mu; \tilde{c}^k, \hat{P}^k) \le V_1^{\pi^0}(\mu; c, P)$ according to Lemma \ref{lem:opt-cost} and $V_1^{\pi^0}(\mu; \tilde{r}, \hat{P}^k) \geq 0$.

On the other hand, if $V_1^{\pi^k}(\mu; \tilde{c}^k, \hat{P}^k) < c_0$, then $\hat{V}_1^{\pi^k}(\mu; \tilde{c}^k, \hat{P}^k) \leq c_0$, we directly have
\begin{align*}
    & \lambda^k \left(\hat{V}_1^{\pi^k}(\mu; \tilde{c}^k, \hat{P}^k) + \epsilon_k - \tau\right) \leq - \frac{\tau - c^0}{2} \lambda^k \leq - \frac{\tau - c^0}{2} \lambda^k + 2 \eta^k H.
\end{align*}
}

For the second term, since $\hat{V}_1^{\pi^0}(\mu; \tilde{c}^k, \hat{P}^k) \le V_1^{\pi^0}(\mu; c, P)$, we have
\begin{align*}
    \frac{1}{2} (\hat{V}_1^{\pi^k}(\mu; \tilde{c}^k, \hat{P}^k) + \epsilon_k - \tau)^2 \leq (\hat{V}_1^{\pi^k}(\mu; \tilde{c}^k, \hat{P}^k) - \tau)^2 + \epsilon_k^2 \leq H^2 + \epsilon_k^2.
\end{align*}
We thus have
\begin{align*}
    L(k+1) - L(k) &\leq - \frac{\tau - c^0}{2} \lambda^k + 2\eta^k H + H^2 + \epsilon_k^2.
\end{align*}

It then follows that $L(k+1) < L(k)$ and
\begin{align*}
    \lambda^{k+1} - \lambda^k & = \frac{(\lambda^{k+1})^2 - (\lambda^k)^2}{\lambda^{k+1} + \lambda^k} \\
    & \leq \frac{(\lambda^{k+1})^2 - (\lambda^k)^2}{2\lambda^k} \\
    & \le - \frac{\tau - c^0}{2} + \frac{H^2 + \epsilon_k^2 + 2\eta^k H}{\varphi_k} = - \frac{\tau - c^0}{4} =: - \rho.
\end{align*}
In addition, since $\forall k \geq C''$, $\epsilon_k \le (\tau - c^0)/2$,
\begin{align*}
    |\lambda^{k+1} - \lambda^k| \leq |\hat{V}_1^{\pi^k}(\mu; \tilde{c}^k, \hat{P}^k) + \epsilon_k - \tau| \leq H =: \nu_{\max}.
\end{align*}
With the definitions of $\varphi_k, \rho, \nu_{\max}$, we apply Lemma \ref{Lyapunov} starting from $C''$, which gives
\begin{align*}
    \lambda^k &\le \frac{1}{\zeta} \ln \left(e^{\zeta \lambda^{C''}} + \frac{2e^{\zeta (v_{\max} + \varphi_k)}}{\zeta \rho}\right) \\
    & \le \frac{1}{\zeta} \ln \left(\frac{11 v_{\max}^2 e^{\zeta(\lambda^{C''} + v_{\max} + \varphi_k)}}{3 \rho^2}\right) \\
    & = \frac{1}{\zeta} \ln \frac{11 \nu_{\max}^2}{3 \rho^2} + \lambda^{C''} + v_{\max} + \varphi_k \\
    & \le \frac{1}{\zeta} \ln \frac{11 \nu_{\max}^2}{3 \rho^2} + C''(H-\tau) + \sum_{u=1}^{C''} \epsilon_u + H + \frac{4(H^2 + \epsilon_k^2 + 2\eta^k H)}{\tau - c^0},
\end{align*}
where $\zeta=\rho/(\nu_{\max }^{2}+\nu_{\max} \rho / 3) \geq 3(\tau - c^0)/(13 H^2)$, and the last inequality is by
\begin{align}
    \lambda^{C''} & \leq \lambda^1 + \sum_{k = 1}^{C''-1} (\hat{V}_1^{\pi^k}(\mu; \tilde{c}^k, \hat{P}^k) + \epsilon_k - \tau)_+ \notag\\
    & \leq \sum_{k =1}^{C''} \epsilon_k + C''(H- \tau).
    \label{eqn:lambda_C''}
\end{align}
\end{proof}

The bounded constraint violation of OptPess-PrimalDual is established as follows. Notice that
\begin{align*}
    \lambda^{K+1} & = \left(\lambda^K + \hat{V}_1^{\pi^K}(\mu; \tilde{c}^K, \hat{P}^K) + \epsilon_{K} - \tau\right)_{+} \\
    & \geq \epsilon_{K} +  \hat{V}_1^{\pi^K}(\mu; \tilde{c}^K, \hat{P}^K)  - \tau + \lambda^K \\
    & \geq \epsilon_K + \epsilon_{K-1} + \hat{V}_1^{\pi^K}(\mu; \tilde{c}^K, \hat{P}^K)  - \tau + \hat{V}_1^{\pi^{K-1}}(\mu; \tilde{c}^{K-1}, \hat{P}^{K-1})  - \tau + \lambda^{K-1} \\
    & \geq \cdots \\
    & \geq \sum_{k = 1}^K \epsilon_k + \sum_{k=1}^K \left(\hat{V}_1^{\pi^k}(\mu; \tilde{c}^k, \hat{P}^k) - \tau \right) + \lambda^1.
\end{align*}
Since $\lambda^1 = 0$ according to the initial condition of Algorithm \ref{alg:dual}, we can decompose the constraint violation:
\begin{align*}
    Reg^{\textbf{OPPD}}(K; c) &= \left(\sum_{k=1}^K \left(V_1^{\pi^k}(\mu; c, P) - \hat{V}_1^{\pi^k}(\mu; \tilde{c}^k, \hat{P}^k)\right) + \sum_{k=1}^K \left(\hat{V}_1^{\pi^k}(\mu; \tilde{c}^k, \hat{P}^k) - \tau\right)\right)_{+} \notag\\
    & \le \left(\sum_{k=1}^K \left(V_1^{\pi^k}(\mu; c, P) - \hat{V}_1^{\pi^k}(\mu; \tilde{c}^k, \hat{P}^k)\right) + \lambda^{K+1} - \sum_{k=1}^K \epsilon_{k} \right)_{+}.
\end{align*}
 
The first summation term can be bounded by Lemma \ref{lem:violation_first_term}. The second $\lambda^{K+1}$ term can be bounded by Lemma \ref{lem7}. Finally, recall $\epsilon_k = 5 H^2 \sqrt{|\Sc|^3 |\Ac|} (\log\frac{k}{\delta'} + 1) / \sqrt{k \log\frac{k}{\delta'}}$, we have
\begin{align*}
    \sum_{k=1}^K \epsilon_k &\geq \int_{1}^{K+1} \epsilon_u \mathrm{d} u \\
    &\geq 10 H^2 \sqrt{|\Sc|^3 |\Ac| K \log \frac{K}{\delta'}} - 10H^2\sqrt{|\Sc|^3 |\Ac| \log \frac{1}{\delta'}}.
\end{align*}
The negative $\sum_{k=1}^K \epsilon_k$ term can cancel the dominant terms. Thus, with probability at least ($1 - \delta$), the violation is bounded as
\begin{align*}
    Reg^{\textbf{OPPD}}(K; c) = \Oc\left(C''H + H^2 \sqrt{|\Sc|^3 |\Ac| C'' \log{(C'' / \delta')}} \right).
\end{align*}

\subsection{Regret of reward analysis}
For episode $k$ with $k \geq C''$, let $\pi^{\epsilon_k, *}$ be the optimal policy for (\ref{pertubed_problem}), which is well-defined by the definition of $C''$. We decompose the reward regret as
\begin{align}
    Reg^{\textbf{OPPD}}&(K; r) = \sum_{k=1}^{C''} \left(V_1^{\pi^*}(\mu; r, P) - V_1^{\pi^k}(\mu; r, P)\right) \notag\\
    + &\sum_{k=C''}^K \left(V_1^{\pi^*}(\mu; r, P) - V_1^{\pi^{\epsilon_k, *}}(\mu; r, P)\right) + \sum_{k=C''}^K \left(V_1^{\pi^{\epsilon_k, *}}(\mu; r, P) - \hat{V}_1^{\pi^{\epsilon_k, *}}(\mu; \tilde{r}^k, \hat{P}^k)\right) \notag\\
    + &\sum_{k=C''}^K \left(\hat{V}_1^{\pi^{\epsilon_k, *}}(\mu; \tilde{r}^k, \hat{P}^k) - \hat{V}_1^{\pi^k}(\mu; \tilde{r}^k, \hat{P}^k)\right) + \sum_{k=C''}^K \left(\hat{V}_1^{\pi^k}(\mu; \tilde{r}^k, \hat{P}^k) - V_1^{\pi^k}(\mu; r, P)\right). \label{bound_regret_decompose}
\end{align}

We upper bound each term in the RHS of (\ref{bound_regret_decompose}). Since $V_1^{\pi}(\mu;r, P) \in [0, H]$ for any policy $\pi$, the first term is upper bounded by $HC''$. The second and third terms can be bounded by the following two lemmas.
\begin{lemma}[Restatement of Lemma \ref{lem3}]
     On the good event $\Ec$,
     \begin{align*}
         \sum_{k=C''}^K \left(V_1^{\pi^*}(\mu; r, P) - V_1^{\pi^{\epsilon_k, *}}(\mu; r, P)\right) \le \sum_{k=C''}^K \frac{\epsilon_k H}{\tau - c^0} = \tilde{\Oc}\left(\frac{H^3}{\tau - c^0} \sqrt{|\Sc|^3 |\Ac| K}\right).
     \end{align*}
\end{lemma}
\begin{proof}[Proof of Lemma \ref{lem3}]
    Define a probabilistic mixed policy
    \begin{align*}
    \pi^{\epsilon_k} = (1 - B_k) \pi^* + B_k \pi^0, \quad \forall k \ge C'',
    \end{align*}
    where $B_k$ is a Bernoulli distributed random variable with mean $\epsilon_k/(\tau-c^0)$.
    
    We have that
    \begin{align*}
        V_1^{\pi^{\epsilon_k}}(\mu; c, P) &= (1 - \frac{\epsilon_k}{\tau - c^0}) V_1^{\pi^*}(\mu; c, P) + \frac{\epsilon_k}{\tau - c^0} V_1^{\pi^0}(\mu; c, P) \\
        &\le (1 - \frac{\epsilon_k}{\tau - c^0}) \tau + \frac{\epsilon_k}{\tau - c^0} c^0 = \tau - \epsilon_k.
    \end{align*}
    Though $\pi^{\epsilon_k}$ is not a Markov policy, by Lemma \ref{lem:richness}, there exists a Markov policy $\hat{\pi}^{\epsilon_k}$, which has the same performance as $\pi^{\epsilon_k}$ under transition probabilities $P$.  $\hat{\pi}^{\epsilon_k}$ is a feasible solution to the following problem
    \begin{align*}
        \max_{\pi} & \quad V_1^{\pi}(\mu; r, P) \\
        \text{s.t.} & \quad V_1^{\pi}(\mu; c, P) + \epsilon_k \le \tau.
    \end{align*}

    We then have 
    \begin{align*}
        \sum_{k=C''}^K V_1^{\pi^*}(\mu; r, P) - V_1^{\pi^{\epsilon_k, *}}(\mu; r, P) & \le \sum_{k=C''}^K V_1^{\pi^*}(\mu; r, P) - V_1^{\hat{\pi}^{\epsilon_k}}(\mu; r, P) \\
        &= \sum_{k=C''}^K V_1^{\pi^*}(\mu; r, P) - V_1^{\pi^{\epsilon_k}}(\mu; r, P)\\
        &\le \sum_{k=C''}^K \frac{\epsilon_k}{\tau - c^0} \left(V_1^{\pi^*}(\mu; r, P) - V_1^{\pi^0}(\mu; r, P) \right)\\
        &\le \sum_{k=C''}^K \frac{\epsilon_k H}{\tau - c^0}.
    \end{align*}
\end{proof}

\begin{lemma}[Restatement of Lemma \ref{lem4}]
    On the good event $\Ec$, $\sum_{k=C''}^K \left(V_1^{\pi^{\epsilon_k, *}}(\mu; r, P) - \hat{V}_1^{\pi^{\epsilon_k, *}}(\mu; \tilde{r}^k, \hat{P}^k)\right) \le 0$.
\end{lemma}

\begin{proof}[Proof of Lemma \ref{lem4}]
For any $k \in [C'', K]$, $\pi^{\epsilon_k, *}$ is well-defined.  If $\hat{V}_1^{\pi^{\epsilon_k, *}}(\mu; \tilde{r}^k, \hat{P}^k) = V_1^{\pi^{\epsilon_k, *}}(\mu; \tilde{r}^k, \hat{P}^k)$, then by the value difference lemma (Lemma \ref{valuediff}), we know
    \begin{align*}
    & V_1^{\pi^{\epsilon_k, *}}(\mu; r, P) - \hat{V}_1^{\pi^{\epsilon_k, *}}(\mu; \tilde{r}^k, \hat{P}^k) \\
    =& \Eb_{\mu, \hat{P}^k, \hat{\pi}^{\epsilon_k, *}} \left[\sum_{h = 1}^H \left(r_h(S_h, A_h) - \tilde{r}_h^k(S_h, A_h) + \sum_{s'} (P_h - \hat{P}^k_h)(s'|S_h, A_h) V^{\pi}_{h+1}(s'; r, P)\right) ~{\Big |}~ \Fc_{k-1} \right] \\
    \leq &\sum_{k=C''}^K \Eb_{\mu, \hat{P}^k, \hat{\pi}^{\epsilon_k, *}} \left[\sum_{h = 1}^H (r_h(S_h, A_h) - \hat{r}_h^k(S_h, A_h)) - \beta_h^k - H |\Sc| \beta_h^k + H |\Sc| \beta_h^k ~{\Big |}~ \Fc_{k-1} \right] \leq 0.
    \end{align*}
Otherwise, $\hat{V}_1^{\pi^{\epsilon_k, *}}(\mu; \tilde{r}^k, \hat{P}^k) = H \geq V_1^{\pi^{\epsilon_k, *}}(\mu; r, P)$. Thus $\hat{V}_1^{\pi^{\epsilon_k, *}}(\mu; \tilde{r}^k, \hat{P}^k) \geq V_1^{\pi^{\epsilon_k, *}}(\mu; r, P)$, for any $k \in [C'', K]$, and the lemma is proved.
\end{proof}

The pivotal step is to leverage optimism of $\pi^k$ to further decompose the fourth term on the RHS of (\ref{bound_regret_decompose}) and utilize the projected dual update to transfer it into the form of $\lambda^k (\lambda^k - \lambda^{k+1})$. The following lemmas provide high probability bounds for the remaining two terms in the RHS of (\ref{bound_regret_decompose}):

\begin{lemma}[Restatement of Lemma \ref{lem5}]
    On the good event $\Ec$,
    \begin{align*}
        \sum_{k=C''}^K \left(\hat{V}_1^{\pi^{\epsilon_k, *}}(\mu; \tilde{r}^k, \hat{P}^k) - \hat{V}_1^{\pi^k}(\mu; \tilde{r}^k, \hat{P}^k)\right) = \tilde{\Oc}\left( \frac{H}{\tau - c^0} \sqrt{K}\right).
    \end{align*}
\end{lemma}
\begin{proof}[Proof of Lemma \ref{lem5}]
    Taking advantage of the optimism of $\pi^k$ and the definition of the dual update, we have the following decomposition:
    \begin{align*}
        & \sum_{k=C''}^K \left(\hat{V}_1^{\pi^{\epsilon_k, *}}(\mu; \tilde{r}^k, \hat{P}^k) - \hat{V}_1^{\pi^k}(\mu; \tilde{r}^k, \hat{P}^k)\right) = \underbrace{\sum_{k=C''}^K \left(\frac{\lambda^k}{\eta^k} (\hat{V}_1^{\pi^{\epsilon_k, *}}(\mu; \tilde{c}^k, \hat{P}^k) - \hat{V}_1^{\pi^k}(\mu; \tilde{c}^k, \hat{P}^k))\right)}_\textnormal{(I)} \\
        +& \underbrace{\sum_{k=C''}^K \left(\hat{V}_1^{\pi^{\epsilon_k, *}}(\mu; \tilde{r}^k, \hat{P}^k) - \frac{\lambda^k}{\eta^k} \hat{V}_1^{\pi^{\epsilon_k, *}}(\mu; \tilde{c}^k, \hat{P}^k)\right) - \left(\hat{V}_1^{\pi^k}(\mu; \tilde{r}^k, \hat{P}^k) - \frac{\lambda^k}{\eta^k} \hat{V}_1^{\pi^k}(\mu; \tilde{c}^k, \hat{P}^k)\right)}_\textnormal{(II)}
    \end{align*}

    Owing to the optimism of $\pi^k$, we know (II) $\le 0$.  Term (I) can then be upper bounded as follows: 
    \begin{align*}
        (\textnormal{I}) & \stackrel{(a)}{\leq} \sum_{k=C''}^K \frac{\lambda^k}{\eta^k} \left(\tau - \epsilon_k - \hat{V}_1^{\pi^k}(\mu; \tilde{c}^k, \hat{P}^k)\right) \\
        & \stackrel{(b)}{\leq} \sum_{k=C''}^K \frac{1}{\eta^k} \left(\lambda^k(\lambda^k - \lambda^{k+1}) + \tau^2 \right) \\
        & \stackrel{(c)}{=} \sum_{k=C''}^K \frac{1}{\eta^k} (\frac{1}{2} (\lambda^k)^2 - \frac{1}{2} (\lambda^{k+1})^2) + \sum_{k=C''}^K \frac{1}{2 \eta^k} (\lambda^{k+1} - \lambda^{k})^2 + \sum_{k = C''}^K \frac{1}{\eta^k}\tau^2 \\
        & \stackrel{(d)}{\leq} \frac{(\lambda^{C''})^2}{2 \eta^{C''}} + \sum_{k=C''}^K \frac{H^2}{2 \eta^k} + \tau^2\sum_{k =C''}^K \frac{1}{\eta^k} \\
        & \stackrel{(e)}{\leq} \frac{(\sum_{k=1}^{C''} \epsilon_k + C''(H - \tau))^2}{2 \eta^{C''}} + \sum_{k=C''}^K \frac{H^2}{2 \eta^k} + \tau^2\sum_{k =C''}^K \frac{1}{\eta^k}.
    \end{align*}
    (a) holds by applying Lemma \ref{lem:opt-cost} for $\pi = \pi^{\epsilon_k, *}$ and $\hat{V}_1^{\pi^{\epsilon_k, *}}(\mu; c, P) \le \tau - \epsilon_k$.

    (b) holds as follows: if $\lambda^{k+1} > 0$, we know
    \begin{align*}
        \tau - \epsilon_k - \hat{V}_1^{\pi^k}(\mu; \tilde{c}^k, \hat{P}^k) = \lambda^k - \lambda^{k+1}.
    \end{align*}
    If $\lambda^{k+1} = 0$, we have 
    \begin{align*}
        \lambda^k \leq \tau - \epsilon_k - \hat{V}_1^{\pi^k}(\mu; \tilde{c}^k, \tilde{P}^k) < \tau,
    \end{align*}
    which indicates that $\lambda^k(\tau - \epsilon_k - \hat{V}_1^{\pi^k}(\mu; \tilde{c}^k, \hat{P}^k)) \leq \tau^2$. Thus,
    \begin{align*}
        \lambda^k(\tau - \epsilon_k - \hat{V}_1^{\pi^k}(\mu; \tilde{c}^k, \hat{P}^k)) \leq \lambda^k(\lambda^k - \lambda^{k+1}) + \tau^2.
    \end{align*}
    Then, (c) holds since $-\epsilon \Delta = \frac{1}{2} \epsilon^2 - \frac{1}{2} (\epsilon + \Delta)^2 + \frac{1}{2} \Delta^2$, where $\epsilon=\lambda^k, \Delta = \lambda^{k+1} - \lambda^k$,
    (d) holds since $\eta^k$ is monotonically increasing and $(\lambda^k - \lambda^{k+1})^2 \le (\hat{V}_1^{\pi^k}(\mu;\tilde{c}^k, \hat{P}^k) + \epsilon_k - \tau)^2 \le H^2$, and (e) holds according to (\ref{eqn:lambda_C''}).
\end{proof}

\begin{lemma}[Restatement of Lemma \ref{lem6}]
    On the good event $\Ec$,
    \begin{align*}
        \sum_{k=C''}^K \left(\hat{V}_1^{\pi^k}(\mu; \tilde{r}^k, \hat{P}^k) - V_1^{\pi^k}(\mu; r, P)\right) = \tilde{\Oc}(H^2 \sqrt{|\Sc|^3 |\Ac| K} + H^4 |\Sc|^3 |\Ac|).
    \end{align*}
\end{lemma} 
\begin{proof}[Proof of Lemma \ref{lem6}]
For any $k, s, a, h$, we have
\begin{align*}
    |(\tilde{r}_h^k - r_h)(s, a)| & = |\tilde{r}_h^k(s, a) - \hat{r}_h^k(s, a) + \hat{r}^k_h(s, a) - r_h(s, a)| \\
    & \leq (1+|\Sc|H)\beta^k_h(s, a) + |\hat{r}^k_h(s, a) - r_h(s, a)| \\
    & \leq (1+|\Sc|H)\beta^k_h(s, a) + \beta^k_h(s, a) = (2+|\Sc|H)\beta^k_h(s, a).
\end{align*}
By Lemma \ref{lem:indi-opt-val-diff}, we have
    \begin{align*}
        \sum_{k = C''}^{K} \left( \hat{V}^{\pi^k}_1(\mu; \tilde{r}^k, \hat{P}^k) - V^{\pi^k}_1(\mu; r, P)\right) \le \tilde{\Oc}(H^2 \sqrt{|\Sc|^3 |\Ac| K} + H^4 |\Sc|^3 |\Ac|).
    \end{align*}
\end{proof}

\subsection{Proof of Theorem \ref{thm2}}
\begin{theorem}[Regret and constraint violation bounds for OptPess-PrimalDual (Restatement of Theorem \ref{thm2})]
    On the good event $\Ec$,
    \begin{align*}
        Reg^{\textbf{OPPD}}(K; r) &= \tilde{\Oc}\left(\frac{H^3}{\tau - c^0} \sqrt{|\Sc|^3 |\Ac| K} + \frac{H^5 |\Sc|^3 |\Ac|}{(\tau - c^0)^2}\right),\\
        Reg^{\textbf{OPPD}}(K; c) &= \Oc\left(C''(H-\tau) + H^2 \sqrt{|\Sc|^3 |\Ac| C''} \right) = \Oc(1),
    \end{align*}
    where $C'' = \Oc (\frac{H^4 |\Sc|^3 |\Ac|}{(\tau - c^0)^2} \log\frac{H^4 |\Sc|^3 |\Ac|}{(\tau - c^0)^2\delta'})$ is a coefficient independent of $K$.
\end{theorem}

\begin{proof}[Proof of Theorem \ref{thm2}]
The bounded constraint analysis has been provided in the previous subsection. Applying Lemmas \ref{lem3}, \ref{lem4}, \ref{lem5}, and  \ref{lem6} yields the regret bound in
Theorem \ref{thm2}.
\end{proof}

\section{Other supporting lemmas}
\begin{lemma}[Hoeffding's inequality \cite{hoeffding1994probability}]
For independent zero-mean $1/2$-sub-Gaussian random variables $X_1, X_2, \ldots, X_n$,
\begin{align*}
    & \Pb\left( \frac{1}{n}\sum_{i=1}^n X_i \geq \epsilon \right) \leq \exp\left(- n \epsilon^2 \right), \\
    & \Pb\left( \frac{1}{n}\sum_{i=1}^n X_i \leq -\epsilon \right) \leq \exp\left(- n \epsilon^2 \right).
\end{align*}
\end{lemma}

\begin{lemma}[Value difference lemma, {\cite[Lemma E.15]{dann2017unifying}}]
    \begin{align*}
    & V_1^{\pi}(\mu; g', P') - V_1^{\pi}(\mu; g, P) \\
    =& \Eb_{\mu, P, \pi} \left[\sum_{h = 1}^H \left(g'(S_h, A_h) - g(S_h, A_h) + \sum_{s'} (P'_h - P_h)(s'|S_h, A_h) V^{\pi}_{h+1}(s'; g', P')\right) ~{\Big |}~ \Fc_{k-1} \right]\\
    =& \Eb_{\mu, P', \pi} \left[ \sum_{h = 1}^H \left(g'(S_h, A_h) - g(S_h, A_h) + \sum_{s'} (P'_h - P_h)(s'|S_h, A_h) V^{\pi}_{h+1}(s'; g, P)\right) ~{\Big |}~ \Fc_{k-1} \right],
    \end{align*}
    where $g = g' = r, c$.
\label{valuediff}
\end{lemma}

\begin{lemma}[{\cite[Theorem 6.1(i)]{altman1999constrained}}] \label{lem:richness}
Suppose the transition function is $P$. For any mixed policy $\tilde{\pi} = B_\gamma \pi^1 + (1 - B_\gamma) \pi^2$, where $B_\gamma$ is a Bernoulli distributed random variable with mean $\gamma$. Then there exists a Markov policy $\hat{\pi}$ that
\begin{align*}
    V^{\hat{\pi}}_{h}(s; r, P) = V^{\tilde{\pi}}_h(s; r, P),\quad \forall r, s, h.
\end{align*}
\end{lemma}

\begin{lemma} \label{lem:indi-opt-val-diff}
Let $\Gc_{1:K}$ be a sequence of events such that $\Gc_{k} \in \Fc_{k-1}$ for each $k \in [K]$. 
Suppose $|\tilde{g}^k - g| \leq \alpha \beta^k$, $\alpha \ge 1$. On the good event $\Ec$, for any $K' \leq K$,
\begin{align*}
    \sum_{k=1}^{K'} \mathbbm{1}(\Gc_k)|V^{\pi^k}_1(\mu;\tilde{g}^k, \hat{P}^{k}) - V^{\pi^k}_1(\mu;g, P)|
    &\leq  (3\alpha + 3\sqrt{2}H\sqrt{|\Sc|})H\sqrt{|\Sc||\Ac|K'_{\Gc} Z} \\
    &+ \tilde{\Oc}\left(\alpha H^3 |\Sc|^2|\Ac|)\right),
\end{align*}
where $\Gc_{1:K} = (\Omega)_{1:K}$ or $\Gc_{1:k} =\{V^{\pi^0}_1(\mu; \underbar{c}^k, \hat{P}^k) \geq \frac{\tau + c^0}{2} \}_{k \in [K]}$, and $K'_{\Gc} = \sum_{k=1}^{K'}\mathbbm{1}(\Gc_k)$.
\end{lemma}
\begin{proof}[Proof of Lemma \ref{lem:indi-opt-val-diff}] 
It can be proved following steps similar to those in the proof of Lemma 32 in \cite{efroni2020exploration}, but utilizing Lemma \ref{lem:indi-series-sum} for taking account of the predictable event sequence $\Gc_{1:K}$, and Lemma \ref{lem:quad-bound} for bounding the leading order explicitly. For completeness, we include the detailed proof here.
\begin{align*}
&\sum_{k=1}^{K'} \mathbbm{1}(\Gc_k)\left|V^{\pi^k}_1(\mu;\tilde{g}^k, \hat{P}^{k}) - V^{\pi^k}_1(\mu;g, P)\right|\\
=& \sum_{k=1}^{K'} \sum_{h=1}^H \sum_{s, a} \mathbbm{1}(\Gc_k) q^{\pi_k}(s, a, h) \left|\tilde{g}_h^k(s, a) - g_h(s, a) + \sum_{s'} (\hat{P}_h^k - P_h)(s'|s, a) V_{h+1}^{\pi_k}(s'; \tilde{g}^k, \hat{P}^k)\right| \\
\stackrel{(a)}{\leq}& \sum_{k=1}^{K'} \sum_{h=1}^H \sum_{s, a} \mathbbm{1}(\Gc_k) q^{\pi_k}(s, a, h) \left[\alpha \beta_h^k(s, a) + \sqrt{\sum_{s'} 2 P(s'|s, a) (\beta_h^k(s, a))^2} \sqrt{|\Sc|H^2} + \frac{ZH|\Sc|}{N_h^k(s, a) \vee 1}\right] \\
+& \sum_{k=1}^{K'} \sum_{h=1}^H \sum_{s, a} \mathbbm{1}(\Gc_k) q^{\pi_k}(s, a, h) \left|\sum_{s'} (\hat{P}_h^k - P_h)(s'|s, a) \left(V_{h+1}^{\pi_k}(s'; \tilde{g}^k, \hat{P}^k) - V_{h+1}^{\pi_k}(s'; g, P)\right)\right|\\
=& \sum_{k=1}^K \sum_{h=1}^H \sum_{s, a} \mathbbm{1}(\Gc_k) q^{\pi_k}(s, a, h) \left[(\alpha + \sqrt{2 |\Sc|}H) \sqrt{\frac{Z}{N_h^k(s, a) \vee 1}} + \frac{ZH|\Sc|}{N_h^k(s, a) \vee 1}\right]\\
+& \sum_{k=1}^{K'} \sum_{h=1}^H \sum_{s, a} \mathbbm{1}(\Gc_k) q^{\pi_k}(s, a, h) \left|\sum_{s'} (\hat{P}_h^k - P_h)(s'|s, a) \left(V_{h+1}^{\pi_k}(s'; \tilde{g}^k, \hat{P}^k) - V_{h+1}^{\pi_k}(s'; g, P)\right)\right|.
\end{align*}
Since
\begin{align*}
&\sum_{k=1}^{K'} \sum_{h=1}^H \sum_{s, a} \mathbbm{1}(\Gc_k) q^{\pi_k}(s, a, h) \left|\sum_{s'} (\hat{P}_h^k - P_h)(s'|s, a) \left(V_{h+1}^{\pi_k}(s'; \tilde{g}^k, \hat{P}^k) - V_{h+1}^{\pi_k}(s'; g, P)\right)\right| \\
\stackrel{(b)}{\leq}& \sum_{k=1}^{K'} \sum_{h=1}^H \sum_{s, a} \mathbbm{1}(\Gc_k) q^{\pi_k}(s, a, h) \sum_{s'} \frac{\sqrt{2 P_h(s'|s, a) Z}}{\sqrt{N_h^k(s, a) \vee 1}} \left|V_{h+1}^{\pi_k}(s'; \tilde{g}^k, \hat{P}^k) - V_{h+1}^{\pi_k}(s'; g, P)\right| \\
+& \sum_{k=1}^{K'} \sum_{h=1}^H \sum_{s, a} \mathbbm{1}(\Gc_k) q^{\pi_k}(s, a, h) \frac{(1+\alpha)|\Sc|HZ}{N_h^k(s, a) \vee 1}\\
\stackrel{(c)}{\leq}& \tilde{\Oc}(\alpha H^2 |\Sc|^2 |\Ac|) + |\Sc| \sqrt{\alpha |\Ac| H^3} \sqrt{\sum_{k=1}^{K'} \mathbbm{1}(\mathcal{G}_k) \left|V_1^{\pi^k}(\mu; \tilde{g}^k, \hat{P}^k) - V_1^{\pi^k}(\mu; g, P)\right|} + |\Sc| \sqrt{\alpha |\Ac| H^3}\\
\times & \sqrt{\sum_{k=1}^{K'} \sum_{h=1}^H \sum_{s, a} \mathbbm{1}(\Gc_k) q^{\pi_k}(s, a, h) \left|\sum_{s'} (\hat{P}_h^k - P_h)(s'|s, a) \left(V_{h+1}^{\pi_k}(s'; \tilde{g}^k, \hat{P}^k) - V_{h+1}^{\pi_k}(s'; g, P)\right)\right|}\\
\stackrel{(d)}{\leq}& \tilde{\Oc}(\alpha H^3 |\Sc|^2 |\Ac|) + \frac{3}{2} |\Sc|H \sqrt{\alpha |\Ac| H} \sqrt{\sum_{k=1}^{K'} \mathbbm{1}(\mathcal{G}_k) \left|V_1^{\pi^k}(\mu; \tilde{g}^k, \hat{P}^k) - V_1^{\pi^k}(\mu; g, P)\right|},
\end{align*}
we have
\begin{align*}
&\sum_{k=1}^{K'} \mathbbm{1}(\mathcal{G}_k) \left|V_1^{\pi^k}(\mu; \tilde{g}^k, \hat{P}^k) - V_1^{\pi^k}(\mu; g, P)\right| \\
\stackrel{(e)}{\leq}& (2\alpha + 2\sqrt{2}H\sqrt{|\Sc|})H\sqrt{|\Sc||\Ac|K'_{\Gc} Z} + \tilde{\Oc}\left(H |\Sc||\Ac| (\alpha + |\Sc|H)\right) + \tilde{\Oc}(\alpha H^3 |\Sc|^2 |\Ac|)\\
+& \frac{3}{2} |\Sc|H \sqrt{\alpha |\Ac| H} \sqrt{\sum_{k=1}^{K'} \mathbbm{1}(\mathcal{G}_k) \left|V_1^{\pi^k}(\mu; \tilde{g}^k, \hat{P}^k) - V_1^{\pi^k}(\mu; g, P)\right|}\\
\stackrel{(f)}{\leq}& (3\alpha + 3\sqrt{2}H\sqrt{|\Sc|})H\sqrt{|\Sc||\Ac|K'_{\Gc} Z} + \tilde{\Oc}\left(\alpha H^3 |\Sc|^2 |\Ac|\right).
\end{align*}
(a) and (b) hold due to the triangle inequality, the Cauchy-Schawarz inequality, and since $|V_1^{\pi^k}(\mu; \tilde{g}^k, \hat{P}^k) - V_1^{\pi^k}(\mu; g, P)| \le |V_1^{\pi^k}(\mu; \tilde{g}^k, \hat{P}^k) - V_1^{\pi^k}(\mu; g, \hat{P}^k)| + |V_1^{\pi^k}(\mu; g, \hat{P}^k) - V_1^{\pi^k}(\mu; g, P)| \le (\alpha+1) H$. (c) follows by steps similar to those in Lemma 32 in \cite{efroni2019tight}. (d) and (f) hold due to Lemma \ref{lem:quad-bound}. (e) holds due to Lemma \ref{lem:indi-series-sum}.
\end{proof}

\begin{lemma}[{\cite[Lemma 10]{jin2020learning}} with a predictable events sequence] \label{lem:indi-series-sum}
Given a sequence of events $\Gc_{1:K}$ that $\Gc_{k} \in \Fc_{k-1}$ for each $k \in [K]$. With probability at least $1 - \delta$, for any $K' \leq K$
\begin{align*}
    \sum_{k = 1}^{K'} \sum_{h =1}^H \sum_{s, a} \frac{\mathbbm{1}(\Gc_k)q^{\pi^k}(s, a, h)}{N^k_h(s, a) \lor 1} & \leq 4H |\Sc||\Ac| + 2H|\Sc||\Ac|\ln K'_{\Gc} + 4\ln\frac{2HK}{\delta}, \\
    \sum_{k = 1}^{K'} \sum_{h =1}^H \sum_{s, a} \frac{\mathbbm{1}(\Gc_k) q^{\pi^k}(s, a, h)}{\sqrt{N^k_h(s, a) \lor 1}} & \leq 6H |\Sc||\Ac| + 2H \sqrt{|\Sc||\Ac| K'_{\Gc}} + 2H |\Sc||\Ac|\ln K'_{\Gc} + 5\ln\frac{2HK}{\delta},
\end{align*}
where $K'_{\Gc} = \sum_{k=1}^{K'}\mathbbm{1}(\Gc_k)$ and $q^{\pi^k}$ is the occupancy measure of policy $\pi^k$, i.e., $q^{\pi^k}(s, a,h) = \Eb_{\mu, P, \pi^k}[\mathbbm{1}(S_h^k=s, A_h^k=a) | \Fc_{k-1}]$.
\end{lemma}
\begin{proof}[Proof of Lemma \ref{lem:indi-series-sum}]
This lemma can be proved following steps similar to those in the proof of Lemma 10 in \cite{jin2020learning}. For completeness, we include the detailed proof here. 

Let $\mathbbm{1}_k(s, a, h)$ be the indicator of whether the pair $(s, a)$ is visited at step $h$ in episode $k$, so that $\Eb_k[\mathbbm{1}_k(s, a, h)] = q^{\pi^k}(s, a, h)$. We decompose the first quantity as
\begin{align*}
    \sum_{k = 1}^{K'} \sum_{h =1}^H \sum_{s, a} \frac{\mathbbm{1}(\Gc_k) q^{\pi^k}(s, a, h)}{N^k_h(s, a) \lor 1} &= \sum_{k = 1}^{K'} \sum_{h =1}^H \sum_{s, a} \frac{\mathbbm{1}(\Gc_k) \mathbbm{1}_k(s, a, h)}{N^k_h(s, a) \lor 1} \\
    &+ \sum_{k = 1}^{K'} \sum_{h =1}^H \sum_{s, a} \frac{\mathbbm{1}(\Gc_k) (q^{\pi^k}(s, a, h) - \mathbbm{1}_k(s, a, h))}{N^k_h(s, a) \lor 1}.
\end{align*}
The first term can be bounded as
\begin{align*}
    \sum_{k = 1}^{K'} \sum_{h =1}^H \sum_{s, a} \frac{\mathbbm{1}(\Gc_k) \mathbbm{1}_k(s, a, h)}{N^k_h(s, a) \lor 1} \le 2 H |\Sc| |\Ac| + H |\Sc| |\Ac| \ln K'_{\Gc}.
\end{align*}
To bound the second term, we apply Lemma 9 in \cite{jin2020learning} with $Y_k = \sum_{h =1}^H \sum_{s, a} \frac{\mathbbm{1}(\Gc_k) (q^{\pi^k}(s, a, h) - \mathbbm{1}_k(s, a, h))}{N^k_h(s, a) \lor 1}, \lambda = 1/2$, and the fact
\begin{align*}
    \Eb_k[Y_k^2] &\le \Eb_k\left[\left(\sum_{h =1}^H \sum_{s, a} \frac{\mathbbm{1}(\Gc_k) \mathbbm{1}_k(s, a, h)}{N^k_h(s, a) \lor 1}\right)^2\right]\\
    &= \Eb_k\left[\sum_{h =1}^H \sum_{s, a} \frac{\mathbbm{1}(\Gc_k) \mathbbm{1}_k(s, a, h)}{(N^k_h(s, a))^2 \lor 1}\right]\\
    & \le \sum_{h =1}^H \sum_{s, a} \frac{\mathbbm{1}(\Gc_k) q^{\pi^k}(s, a, h)}{N^k_h(s, a) \lor 1},
\end{align*}
which show that with probability at least $1 - \delta/(2 H K)$,
\begin{align*}
    \sum_{k = 1}^{K'} \sum_{h =1}^H \sum_{s, a} \frac{\mathbbm{1}(\Gc_k) (q^{\pi^k}(s, a, h) - \mathbbm{1}_k(s, a, h))}{N^k_h(s, a) \lor 1} \le \frac{1}{2} \sum_{k = 1}^{K'} \sum_{h =1}^H \sum_{s, a} \frac{\mathbbm{1}(\Gc_k) q^{\pi^k}(s, a, h)}{N^k_h(s, a) \lor 1} + 2 \ln \frac{2HK}{\delta}.
\end{align*}
Combining these two bounds, rearranging, and applying a union bound over $k$ prove the first part of the lemma.

Similarly, we decompose the second quantity as
\begin{align*}
     \sum_{k = 1}^{K'} \sum_{h =1}^H \sum_{s, a} \frac{\mathbbm{1}(\Gc_k) q^{\pi^k}(s, a, h)}{\sqrt{N^k_h(s, a) \lor 1}} &= \sum_{k = 1}^{K'} \sum_{h =1}^H \sum_{s, a} \frac{\mathbbm{1}(\Gc_k) \mathbbm{1}_k(s, a, h)}{\sqrt{N^k_h(s, a) \lor 1}}\\ 
     &+ \sum_{k = 1}^{K'} \sum_{h =1}^H \sum_{s, a} \frac{\mathbbm{1}(\Gc_k) (q^{\pi^k}(s, a, h) - \mathbbm{1}_k(s, a, h))}{\sqrt{N^k_h(s, a) \lor 1}}.
\end{align*}
The first term is bounded as
\begin{align*}
     \sum_{k = 1}^{K'} \sum_{h =1}^H \sum_{s, a} \frac{\mathbbm{1}(\Gc_k) \mathbbm{1}_k(s, a, h)}{\sqrt{N^k_h(s, a) \lor 1}} \le 2 H |\Sc| |\Ac| + 2 H \sqrt{|\Sc| |\Ac| K'_{\Gc}}.
\end{align*}
To bound the second term, we again apply Lemma 9 in \cite{jin2020learning} with $Y_k =  \sum_{h =1}^H \sum_{s, a} \frac{\mathbbm{1}(\Gc_k)(q^{\pi^k}(s, a, h) - \mathbbm{1}_k(s, a, h))}{\sqrt{N^k_h(s, a) \lor 1}} \le 1, \lambda = 1$, and the fact
\begin{align*}
    \Eb_k[Y_k^2] \le \Eb_k\left[\left(\sum_{h =1}^H \sum_{s, a} \frac{\mathbbm{1}(\Gc_k) \mathbbm{1}_k(s, a, h)}{\sqrt{N^k_h(s, a) \lor 1}}\right)^2\right] = \sum_{h =1}^H \sum_{s, a} \frac{\mathbbm{1}(\Gc_k) q^{\pi^k}(s, a, h)}{N^k_h(s, a) \lor 1},
\end{align*}
which shows that with probability at least $1 - \delta/(2HK)$,
\begin{align*}
    \sum_{k = 1}^{K'} \sum_{h =1}^H \sum_{s, a} \frac{\mathbbm{1}(\Gc_k) (q^{\pi^k}(s, a, h) - \mathbbm{1}_k(s, a, h))}{\sqrt{N^k_h(s, a) \lor 1}} &\le \sum_{k=1}^{K'} \sum_{h=1}^H \sum_{s, a} \frac{\mathbbm{1}(\Gc_k) q^{\pi^k}(s, a, h)}{N^k_h(s, a) \lor 1} + \ln \frac{2HK}{\delta}.
\end{align*}
Combining the first part of the lemma and employing a union bound prove the second part of the lemma.
\end{proof}

\begin{lemma}\label{lem:quad-bound}
Suppose $0 \leq x \leq a + b \sqrt{x}$, for some $a, b > 0$,
\begin{align*}
    x \leq \frac{3}{2} a + \frac{3}{2}b^2.
\end{align*}
\end{lemma}
\begin{proof}[Proof of Lemma \ref{lem:quad-bound}]
By solving the equation of $x_0 = a + b \sqrt{x_0}$, we know
\begin{align*}
    x & \leq x_0 = \frac{1}{2}(\sqrt{4ab^2 + b^4} + 2a + b^2) \\
    & \leq \sqrt{ab^2} + \frac{b^2}{2} + a + \frac{b^2}{2} \\
    & \leq \frac{3}{2}a + \frac{3}{2}b^2.
\end{align*}
\end{proof}

\begin{lemma}[{\cite[Lemma 11]{liu2021efficient}}]
    Let $S(k)$ be a random process, $\Phi(k)$ be its Lyapunov function with $\Phi(0)=\Phi_{0}$ and $\Delta(k)=\Phi(k+1)-\Phi(k)$ be the Lyapunov drift. Given an increasing sequence $\left\{\varphi_{t}\right\}, \rho$ and $\nu_{\max }$ with $0<\rho \le \nu_{\max}$, if the expected drift $\mathbb{E}[\Delta(k)|S(k)=s]$ satisfies the following conditions: \\
    (i) There exists constants $\rho>0$ and $\varphi_{t}>0$ s.t. $\Delta(k) \le-\rho$ when $\Phi(k) \ge \varphi_{k}$, and \\
    (ii) $|\Phi(k+1)-\Phi(k)| \le \nu_{\max}$ holds with probability one;\\
    then we have
    \begin{align*}
    e^{\zeta \Phi(k)} \le e^{\zeta \Phi_{0}}+\frac{2 e^{\zeta\left(\nu_{\max }+\varphi_{k}\right)}}{\zeta \rho},
    \end{align*}
    where $\zeta=\rho/(\nu_{\max }^{2}+\nu_{\max} \rho / 3)$.
    \label{Lyapunov}
\end{lemma}

\section{Extensions (unknown $c^0$ with prior knowledge of $\pi^0$)}
\label{app:extension}
It may be noticed that the OptPess-LP algorithm can guarantee zero violation with high probability if a coefficient $c^{0'}$ with $V^{\pi^0}_1(\mu; c, P) \leq c^{0'} < \tau$ is used to replace $c^0$ in the inputs. Thus when the value $V^{\pi^0}_1(\mu; c, P)$ is not known, one may first estimate an appropriate upper bound of $V^{\pi^0}_1(\mu; c, P)$, and call OptPess-LP using this upper bound. The details of choosing the upper bound are as follows.

We can execute policy $\pi^0$ sequentially until an estimate of $c^0$ with enough precision is obtained. Denote $\hat{c}^{0}(k)$ as the mean estimate after executing policy $\pi^0$ $k$ times. We stop the estimation process after $K''$ executions, where
\begin{align*}
    K'' := \min\left\{k \geq 1: \tau - \hat{c}^0(k) \geq 3 \sqrt{\frac{1}{k H}\log\frac{2K}{\delta''}} \right\} \land K.
\end{align*}
By Hoeffding's inequality, with probability at least $1 - \delta''$,
\begin{align*}
    |\hat{c}^{0}(k) - c^0| \leq \sqrt{\frac{1}{k H}\log\frac{2K}{\delta''}}.
\end{align*}

Thus with probability at least $1 - \delta''$, for $k = K'' -1$,
\begin{align*}
    \tau - 4\sqrt{\frac{1}{k H}\log\frac{2K}{\delta''}} < \hat{c}^0(k) - \sqrt{\frac{1}{k H}\log\frac{2K}{\delta''}} \leq c^0,
\end{align*}
which implies
\begin{align*}
    K'' = k + 1 \leq 1 + \frac{16 \log(2K/\delta'')}{H(\tau - c^0)^2}.
\end{align*}
Then $K'' < K$ can then be guaranteed if $\frac{16 \log(2K/\delta'')}{H(\tau - c^0)^2} < K - 1$. After $K''$ executions, with probability at least $1-\delta''$,
\begin{align*}
    c^{0'}=&\hat{c}^0(K'') + \sqrt{\frac{1}{K'' H}\log\frac{2K}{\delta''}} \\
    =& \frac{1}{2}\left(\hat{c}^0(K'') - \sqrt{\frac{1}{K'' H}\log\frac{2K}{\delta''}} \right) + \frac{1}{2}\left(\hat{c}^0(K'') + 3\sqrt{\frac{1}{K'' H}\log\frac{2K}{\delta''}} \right) \\
    \leq& \frac{1}{2} c^0 + \frac{1}{2} \tau = \frac{1}{2}(c^0 + \tau).
\end{align*}
We can apply the OptPess-LP algorithm with a pessimistic estimate $\left( \hat{c}^0(K'') + \sqrt{\frac{1}{K'' H}\log\frac{2K}{\delta''}}\right)$ of $c^0$. It also maintains zero constraint violation with the same order of regret since $1 / (\tau - c^{0'}) \le 1 / (\tau - \frac{1}{2}(c^0 + \tau)) = 2 / (\tau - c^0)$.

\end{document}